\newif\ifarxiv
\arxivtrue

\documentclass{article}

\ifarxiv
\usepackage{float}
\usepackage{xcolor}
\usepackage{fullpage, multicol}
\date{}
\usepackage{natbib}

\else

\usepackage[main, preprint]{neurips_2026}

\fi

\usepackage[utf8]{inputenc} 
\usepackage[T1]{fontenc}    
\usepackage{hyperref}       
\usepackage{url}            
\usepackage{booktabs}       
\usepackage{amsfonts}       
\usepackage{nicefrac}       
\usepackage{microtype}      
\usepackage{xcolor}  
\usepackage[dvipsnames]{xcolor}
\usepackage[most]{tcolorbox}
\usepackage{amsmath,amsthm,amssymb,bbm}
\usepackage{mathtools}
\usepackage{thmtools, thm-restate}

\usepackage{algorithm}
\usepackage{algpseudocode}
\usepackage{listings}

\usepackage{subcaption}
\usepackage{graphicx}
\usepackage{float}
\usepackage{enumitem}
\usepackage{wrapfig}

\newcommand{\cX}{\mathcal{X}}

\newcommand{\cA}{\mathcal{A}}

\newcommand{\E}{\mathbb{E}}

\DeclareFontFamily{OT1}{pzc}{}
\DeclareFontShape{OT1}{pzc}{m}{it}{<-> s * [1.10] pzcmi7t}{}
\DeclareMathAlphabet{\pzccal}{OT1}{pzc}{m}{it}

\theoremstyle{plain}
\newtheorem{theorem}{Theorem}[section]
\newtheorem*{theorem*}{Theorem}
\newtheorem{proposition}[theorem]{Proposition}
\newtheorem{lemma}[theorem]{Lemma}

\newtheorem{observation}[theorem]{Observation}
\theoremstyle{definition}
\newtheorem{definition}[theorem]{Definition}
\newtheorem{assumption}[theorem]{Assumption}
\theoremstyle{remark}
\newtheorem{remark}[theorem]{Remark}

\newcommand{\mh}[1]{\textcolor{violet}{[MH: #1]}}

\newcommand{\rebuttal}[1]{\textcolor{black}{#1}}

\providecommand{\sgn}{\operatorname{sign}}
\title{Narrowing the Collaboration Gap, Probably}

\ifarxiv

\author{%
    \begin{tabular}{ccc}
    \begin{tabular}{c}
    Mirah Shi\thanks{Equal contribution. Correspondence to mirahshi@seas.upenn.edu or mhussing@mit.edu} \\
    University of Pennsylvania
    \end{tabular}
    &
    \begin{tabular}{c}
    Marcel Hussing\footnotemark[1] \\
    MIT 
    \end{tabular}
    &
    \begin{tabular}{c}
    Ira Globus-Harris \\
    Cornell University
    \end{tabular}
    \\[2.5em]
    \begin{tabular}{c}
    Natalie Collina \\
    MIT
    \end{tabular}
    &
    \begin{tabular}{c}
    Aaron Roth \\
    University of Pennsylvania
    \end{tabular}
    &
    \begin{tabular}{c}
    Surbhi Goel \\
    University of Pennsylvania
    \end{tabular}
    \end{tabular}
    }
\else

\author{%
  Mirah Shi\thanks{Equal contribution} \\
  University of Pennsylvania\\
  mirahshi@seas.upenn.edu\\
  \And
  Marcel Hussing\footnotemark[1]\\
  MIT\\
  mhussing@seas.upenn.edu\\
  \And
  \;\;\;\;\;Ira Globus-Harris \;\;\;\;\;\\
  Cornell University\\
  img53@cornell.edu\\
  \AND
  Natalie Collina \\
  MIT\\
  ncollina@seas.upenn.edu\\
  \And
  Aaron Roth \\
  University of Pennsylvania\\
  aaroth@seas.upenn.edu\\
  \And
  Surbhi Goel \\
  University of Pennsylvania\\
  surbhig@seas.upenn.edu\\
}
\fi

\begin{document}

\maketitle

\begin{abstract}
Large language models are increasingly deployed as teams of agents that hold different private information about a shared task. How should such agents concisely communicate in a way that promotes aggregation of decision-relevant information? Communicating only proposed actions can discard decision-relevant uncertainty: two agents may prefer the same action for different reasons. We study an alternative interface in which agents communicate calibrated beliefs over actions. We first analyze a simple synthetic setting that captures the structure of multi-round reasoning under partial information. In this setting, we show theoretically and empirically that communicating unbiased probabilities can be strictly more powerful than communicating actions. We then test the same principle in a collaborative maze-solving task with trained transformers and pretrained LLM agents. Our results demonstrate that probability communication outperforms action communication when the exchanged predictions preserve calibrated uncertainty. When communicated probabilities are biased, we show that the benefits of probability communication can disappear. We give a post-hoc conversational calibration intervention that consistently improves decision-making by correcting these biases both in a toy transformer setting and when pretrained LLMs use the corrections in context. These results suggest that probabilistic communication is useful for multi-agent collaboration not merely because it transmits more information, but because calibrated uncertainty gives other agents a reliable object to update on.
\end{abstract}

\section{Introduction}

Large language models (LLMs) are increasingly used within multi-agent systems where agents may share a task but hold different private information.
For example, physical robots may see different sensor data, and coding agents may inspect different files when the relevant repository is too large for any single context window. To succeed in such settings, agents must communicate to combine their private information to make better decisions than any single agent could alone.

In practice, the same constraints that motivate multi-agent design often make full information sharing infeasible.
Communication may be limited by bandwidth, memory, context length, or by the fact that different agents' observations are not mutually legible across modalities. And even if one agent could transmit its full observation, the receiving agent may not be able to process it into a correct answer \citep{davidson2025collaboration}. More generally, recent work has shown that LLMs often struggle in these collaborative settings, exhibiting failures in information sharing~\citep{cemri2026why}, belief updating~\citep{li2023theory, wilie2024belief}, and joint decision-making under partial information~\citep{davidson2025collaboration,eisenstein2026mtpingevalevaluatingmultiturncollaboration}. These failures reveal a broader ``collaboration gap" \citep{davidson2025collaboration}: current LLM agents do not reliably aggregate distributed observations into well-coordinated decisions. Narrowing this gap requires answering a fundamental question:

\begin{center}
    \emph{What should agents communicate in order to aggregate distributed private information?}
\end{center}

A natural low-bandwidth alternative to communicating their hidden information is for agents to exchange only their proposed actions. However, communicating actions can omit decision-relevant uncertainty. A promising direction comes from the agreement literature, which shows that Bayesian reasoners with different observations can exchange beliefs to converge towards predictions that aggregate their information \citep{frongillo2023agreement}.
More recent work weakens Bayesian rationality to calibration-style conditions, including \emph{conversation calibration}, and studies when agreement would lead to optimal actions given the agents' combined observations~\citep{collina2025tractable,collina2026collaborative,kearns2026networked}. This perspective suggests that reliable multi-agent LLM collaboration may require mechanisms that go beyond ad hoc message exchange: agents must be able to recognize their own uncertainty and interpret others' uncertainty.

\begin{figure}[t]
    \centering
    \includegraphics[width=\textwidth]{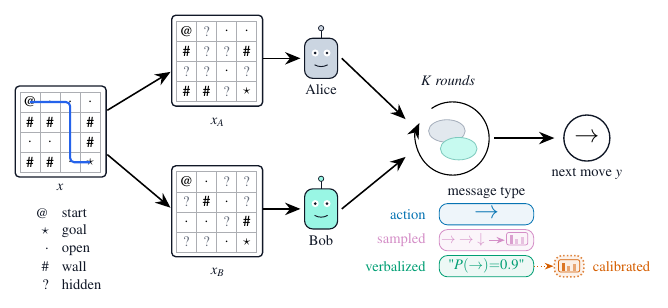}
    \caption{Multi-agent collaboration under partial information. Agents observe different views of the same maze and exchange messages over multiple rounds before producing a final action. We compare action-only messages, verbalized probabilities, sampled probabilities from multiple rollouts, and a post-hoc calibrated probability channel. \textit{Figure adapted from \cite{davidson2025collaboration}.}}
    \label{fig:communication-protocols}
\end{figure}

\subsection{Our contributions}
Our hypothesis is that, under partial information, agents collaborate more effectively when they communicate calibrated beliefs rather than only their preferred decisions. We investigate this hypothesis theoretically in a stylized theoretical model and empirically on the collaborative maze task of~\citet{davidson2025collaboration}. We study both transformers trained for the task and large-scale, pretrained LLM agents. Across these settings, we compare two low-bandwidth communication protocols:
\begin{enumerate}[leftmargin=15pt, topsep=0pt]
\item\textit{Action-only communication:} In each round, the speaker sends a single action that it currently believes is most likely correct. For transformers this is the action with highest softmax probability; for LLM agents, it is the action selected in the model's response.

\item\textit{Probability communication:} In each round, the speaker sends a distribution over actions. For transformers this is the softmax vector. For LLM agents, we consider \emph{verbalized probabilities}, where the model reports a probability vector in text, and \emph{sampled probabilities}, where the message is the empirical distribution over independent response rollouts.
\end{enumerate}

Our results are threefold. First, in a stylized chained-parity model of multi-round collaboration, we show a theoretical separation between communicating probabilities and communicating actions. Exact probability messages aggregate one new piece of information per round, ultimately leading to full information aggregation. Meanwhile, action-only messages can reveal only the first agent's information and then stall. Second, we test the same principle in controlled transformer experiments on collaborative maze solving (see Figure~\ref{fig:communication-protocols}). We see that probability communication improves over action communication when the learned messages preserve calibrated uncertainty. When probability messages are \textit{biased}, this advantage can disappear. To address this, we introduce a \textit{post-hoc calibration algorithm} that corrects biased reports and recovers much of the performance lost when raw probability reports are distorted. Third, we evaluate the same communication protocols using pretrained LLM agents on the maze task. We find that simply asking models to verbalize probabilities is not enough: raw verbalized probabilities are often no better than action messages. We show that across the tested models, applying post-hoc calibration as an \textit{intervention} consistently improves performance, indicating that probabilistic messages become substantially more useful in-context when corrected into forecasts that other agents can reliably update on.

Beyond the agreement literature discussed above, our work also draws on research on calibrated forecasting, uncertainty elicitation in language models, sampling-based uncertainty, multi-agent LLM debate and collaboration, and learned communication in multi-agent reinforcement learning. We defer an extended discussion of these connections to Appendix~\ref{app:additional-related-work}.

\section{Preliminaries}

Formally, we study a setting with two agents, Alice and Bob, who collaborate to find a correct action to take. Each data point consists of a feature $x\in\cX$ and its associated correct action $a^*\in \cA$, drawn from a joint distribution over $\cX\times \cA$. Neither agent observes the full feature: Alice observes $x_A \in \cX_A$ and Bob observes $x_B \in \cX_B$, with $x=(x_A,x_B)\in\cX_A\times\cX_B$. The correct action depends on the full feature $x$, so neither agent can reliably identify $a^*$ from their own observations alone. In the special case where $a^*$ is deterministically determined by $x$, the agents' combined information is sufficient to identify $a^\star$.

\textbf{Communication protocol.}
The agents communicate for a fixed number of rounds $K$ before selecting an action. Alice speaks on odd rounds and Bob on even rounds; we write $\sigma(k)\in\{A,B\}$ for the round-$k$ speaker. We focus on the bandwidth-limited regime, in which each round's message is much smaller than the speaker's private observation.

We consider two types of communication protocols. In the \emph{action-communication} protocol, the round-$k$ message is a single proposed action $\hat a_k\in\cA$. In the \emph{probability-communication} protocol, the round-$k$ message is a distribution over actions $\hat p_k\in\Delta(\cA)$. After round $K$, the agents execute the action chosen by the final speaker: the final action message $\hat a_K$ under action communication, and $\arg\max_a \hat p_K(a)$ under probability communication.

\textbf{Calibration.}
Calibration formalizes when a probability message can be interpreted as meaningful uncertainty. Intuitively, a probabilistic message is calibrated if its probabilities agree with long-run frequencies. For example, among all instances in which an agent assigns probability $0.7$ to an action $a$, that action should be correct approximately $70\%$ of the time.

\begin{definition}[Calibration] \label{def:calibration}
    Let $\hat p$ be a $\Delta(\cA)$-valued random variable, where the randomness is over the draw of the instance and any randomness in the agents' procedure. We say $\hat p$ is calibrated if, for every action $a\in\cA$, $\Pr(a^\star=a\mid \hat p)=\hat p(a)$ almost surely.
\end{definition}

Calibration is a marginal property: it says that an agent's stated uncertainty is correct in aggregate over all instances in which the agent sends the same distribution. With multiple agents, we may also ask if one agent's forecast remains calibrated after accounting for another communicated forecast. This is the notion of cross-calibration. Intuitively, Alice's forecast is cross-calibrated with Bob's forecast if, among all instances where Alice sends a particular distribution, subdividing those instances according to Bob's distribution does not change the correctness of Alice's probabilities.

\begin{definition}[Cross calibration] \label{def:cross_calibration}
    Let $\hat p$ and $\hat q$ be $\Delta(\cA)$-valued random variables. We say $\hat p$ is cross-calibrated with respect to $\hat q$ if, for every action $a\in\cA$,
    $\Pr(a^\star=a\mid \hat p,\hat q)=\hat p(a)$
    almost surely.
\end{definition}

Calibration requires an agent's probabilities to be correct in aggregate, while cross-calibration requires them to remain correct even after conditioning on another communicated forecast. This definition conditions only on the communicated probabilities, not on the other agent's private observation. 

Conversation calibration instantiates cross calibration in a communication protocol, requiring forecasts in every round to be cross-calibrated with respect to forecasts of the previous round.

\begin{definition}[Conversation calibration] \label{def:conversation_calibration}
    Following~\citet{collina2025tractable}, in a probability-communication protocol, the conversation is calibrated through round $K$ if the initial forecast $\hat p_1$ is calibrated and, for every round $2\leq k\leq K$, the round-$k$ forecast is cross-calibrated with respect to the previous round's forecast:
        $\Pr(a^\star=a\mid \hat p_k,\hat p_{k-1})
        =
        \hat p_k(a)
        \text{ for all } a\in\cA$ almost surely.
\end{definition}

\section{When do probabilities outperform actions?}\label{sec:prob-vs-action}

This section establishes when probability communication can be more useful than exchanging only actions. The gap appears when the best action is too coarse to reveal which evidence state produced it, which we demonstrate via a multi-round separation using a \textit{chained parities} task. A warm-up example for one-round, where two posterior states induce the same best action but different probabilities, is included in Appendix \ref{ap:warmup}.

The chained-parity task is intended as a stylized abstraction of general multi-round reasoning problems. In particular, many reasoning problems can be represented as a computation over a chain or graph of interdependent states: each state represents an intermediate reasoning step that incorporates new evidence towards solving the problem. This structure appears, for instance, in multi-hop or inductive reasoning, where intermediate facts are composed along a graph \citep{hotpotqa, musique, clutrr}; and in proof-style reasoning, where intermediate claims support later ones \citep{proofwriter, entailmenttree}. In the chained-parity task, each ``prefix parity" represents an intermediate state, and the next bit in a bitstring is the evidence needed to advance the chain (Figure \ref{fig:chained-parity-task}). We are interested in such problems in collaborative settings, where no single agent observes all information needed to perform the reasoning task, and each agent refines its prediction by using information supplied by the other. 

We will show that probability messages can preserve a sufficient summary of the accumulated state, allowing each round to incorporate one new piece of private information, whereas action messages can collapse intermediate states and thereby prevent later information from being aggregated.

\noindent\textbf{Chained-parity task.}
We now define the task. We consider a binary action space $\cA=\{0,1\}$. Our task is parameterized by the number of rounds $K\geq 1$; for each value of $K$, the joint information space consists of $K$ independent uniform bits $b_1,\ldots,b_K$ in $\{-1,+1\}$. Let $d_A:=\lceil K/2\rceil$ and $d_B:=\lfloor K/2\rfloor$. Alice and Bob observe the odd- and even-indexed bits as fixed-length vectors:
\begin{equation*}
    x_A := (b_1,b_3,\ldots,b_{2d_A-1})\in\{-1,+1\}^{d_A},
    \quad
    x_B := (b_2,b_4,\ldots,b_{2d_B})\in\{-1,+1\}^{d_B}.
\end{equation*}
\label{sec:chained-parities}
\ifarxiv
\begin{figure}[t]
    \centering
    \includegraphics[width=.35\linewidth]{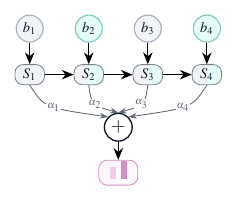}
    \caption{Chained-parity dependency structure.}
    \label{fig:chained-parity-task}
\end{figure}
\fi
Define the \emph{prefix parity}, the $\pm 1$-valued product (equivalently, the XOR) of the first $j$ bits,
\begin{equation*}
    S_0:=1,
    \qquad
    S_j:=\prod_{\ell=1}^j b_\ell
    \quad \text{for } j=1,\ldots,K.
\end{equation*}
Conditional on $x=(x_A,x_B)$, the target action $a^\star$ is a Bernoulli random variable with parameter
\begin{equation*}
    p^\star := \frac{1}{2}+\sum_{j=1}^K \alpha_j S_j,
\end{equation*}
which we call the \emph{full-information posterior probability} of $a^\star=1$.  Throughout this section, $p^\star$ denotes this random variable, a deterministic function of the bit vector $b=(b_1,\ldots,b_K)$. The weights $\alpha_1,\ldots,\alpha_K$ are positive and satisfy $\sum_{j=1}^K\alpha_j<1/2$, which keeps $p^\star\in(0,1)$. Define $h^\star:=\mathbf 1\{p^\star\geq\tfrac12\}\in\{0,1\}$, the Bayes-optimal action under $0$--$1$ loss.  Because each term $\alpha_j S_j$ with $j\geq 2$ depends on \emph{both} agents' bits (Figure~\ref{fig:chained-parity-task}), no agent can evaluate $p^\star$ alone.

\noindent\textit{Communication protocol.} At round $k$, Alice speaks if $k$ is odd and Bob speaks if $k$ is even; write this speaker as $\sigma(k)\in\{A,B\}$. Thus the round-$k$ speaker observes $b_k$. When Alice and Bob communicate probabilities, define $p_0:=1/2$ and let the speaker's Bayes posterior probability that $a^\star=1$, given their own features and the previous probability message, be
\begin{equation*}
    p_k := \Pr\big(a^\star=1 \,\big|\, x_{\sigma(k)},\,p_{k-1}\big) \in [0,1].
\end{equation*}
The probability communicated is $\hat p_k=p_k$. When actions are communicated instead, let $\tilde p_k$ be the analogous Bayes posterior computed from the speaker's own features and the public action transcript so far; the speaker's  action that they communicate to the other player will then be their best response,
\begin{equation*}
    \hat a_k = \mathbf 1\{\tilde p_k\geq\tfrac{1}{2}\} \in \{0,1\}.
\end{equation*}
The distinction is that $p_k$ can preserve how strongly the speaker favors an action, while $\hat a_k$ keeps only the sign of that preference. Later, we also consider how information aggregation is affected when the communicated probability $\hat p_k$ is biased transformation of the Bayes posterior $p_k$. 

\paragraph{Information aggregation with probabilities.}

We now add the condition that makes the previous probability message a sufficient scalar summary of the parities revealed so far.
\begin{assumption}
\label{ass:coefficients}
Distinct bit prefixes $(S_1,\ldots,S_k)\in\{-1,+1\}^k$ yield distinct partial sums $\sum_{j\leq k}\alpha_j S_j$ for each $k\leq K$ (e.g.\ $\alpha_j=2^{-j-1}$).
\end{assumption}

The distinct-partial-sums condition ensures that a scalar probability message identifies the prefix parities already revealed. Therefore the previous probability message is not merely a confidence score; it is an invertible summary of the information accumulated so far. This lets the probability exchange advance one message at a time:
\begin{equation*}
    p_k=p_{k-1}+\alpha_k S_k
    =\frac12+\sum_{j\leq k}\alpha_j S_j.
\end{equation*}
By Assumption~\ref{ass:coefficients}, the previous scalar message $p_{k-1}$ identifies the prefix $(S_1,\ldots,S_{k-1})$; the round-$k$ speaker also observes $b_k$, and hence determines $S_k=S_{k-1}b_k$. All later parity terms contain an unseen bit and have conditional mean zero. Since $p_k$ is a Bayes posterior conditioned on the speaker's information and the previous probability message, the probabilities communicated satisfy conversation calibration in the sense of Definition~\ref{def:conversation_calibration}.

\begin{restatable}[Probability exchange aggregates information]{theorem}{chainedParityPred}
\label{thm:prediction-aggregates}
In the chained-parity model, under Assumption~\ref{ass:coefficients}, when probabilities are communicated, the round-$k$ message equals $p_k = \tfrac{1}{2}+\sum_{j\leq k}\alpha_j S_j$ for every $k\leq K$; in particular $p_K=p^\star$.  The round-$k$ excess squared loss $\mathbb E[(a^\star-p_k)^2]-\mathbb E[(a^\star-p^\star)^2]$ equals $\sum_{j>k}\alpha_j^2$ and decreases by $\alpha_k^2$ each round, so after $K$ rounds the agents attain the full-information Bayes risk $\mathbb E[(a^\star-p^\star)^2]$.
\end{restatable}
\textit{Proof sketch.}
The recurrence above gives $p_k=1/2+\sum_{j\leq k}\alpha_j S_j$ by induction, so $p_K=p^\star$. At any $k<K$,  $p_k$ differs from $p^\star$ by the unrevealed tail $\sum_{j>k}\alpha_j S_j$. Distinct parities are orthogonal, so the excess squared loss is $\mathbb E[(p^\star-p_k)^2]=\sum_{j>k}\alpha_j^2$, and the one-round improvement is $\alpha_k^2$.

\noindent\textbf{Information loss with actions.} By contrast, the action exchange can collapse all later evidence states to the same first action.
\begin{theorem}[Action exchange transmits only the first bit]
\label{thm:action-stalls}
In the chained-parity model, assume $K\geq 2$ and $\alpha_1>\alpha_2$. When actions are exchanged, every round-$k$ message equals $\mathbf 1\{b_1=+1\}$, so the induced decision is constant in $k$.  In particular, for $K=3$ and $\alpha=(0.20,0.15,0.10)$, this decision disagrees with the full-information Bayes action $h^\star$ on an event of probability $1/4$.
\end{theorem}

\ifarxiv
\begin{wrapfigure}{r}{0.36\textwidth}
    \centering
    \vspace{-10pt}
    \includegraphics[width=0.36\textwidth]{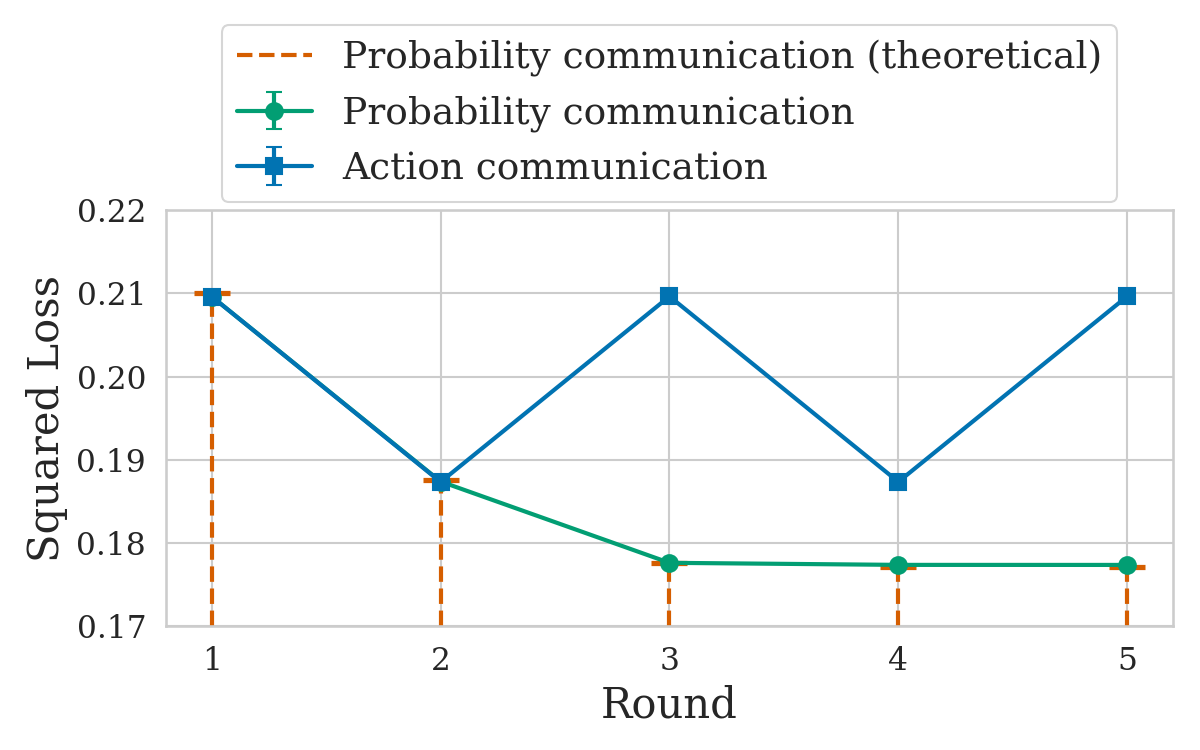}
    \caption{Squared loss of  action vs probability exchange over $5$ rounds on chained parities. Learned probability communication attains the theoretical optimum.}
    \label{fig:parities-sq-err}
    \vspace{-12pt}
\end{wrapfigure}
\fi
\textit{Proof sketch.} The action message keeps only the sign of the speaker's current posterior and discards the magnitude. Alice's first posterior is $1/2+\alpha_1 b_1$, so her action message is $\mathbf 1\{b_1=+1\}$. If Alice speaks later, she still lacks $b_2$, so her posterior again has sign $b_1$. If Bob speaks, he can use $b_2$, but later terms average out and his posterior is $1/2+\alpha_1 b_1+\alpha_2 b_1b_2=1/2+b_1(\alpha_1+\alpha_2 b_2)$. Since $\alpha_1>\alpha_2$, this posterior also has sign $b_1$. Thus every action message repeats the first bit and no later bit is transmitted. For the displayed $K=3$ coefficients, on the event $\{b_2=-1,b_3=+1\}$, the full posterior satisfies $p^\star-1/2=b_1(\alpha_1-\alpha_2-\alpha_3)<0$, so $h^\star$ is the opposite of the action-channel decision; this event has probability $1/4$.

\noindent\textbf{Empirical validation.} We validate this separation with trained transformers on the same chained-parity task (Appendix~\ref{app:exp_details} for details). Figure~\ref{fig:parities-sq-err} shows learned probability messages track the theoretical loss curve, while learned action messages stall as predicted.

\section{When do probability reports need calibration?}\label{sec:calibration}

Theorem~\ref{thm:prediction-aggregates} assumes the round-$k$ speaker transmits its Bayes posterior $p_k$ exactly. In practice, reported probabilities may be distorted.
One way to model this failure is to combine two deviations that are well-documented in LLMs: global miscalibration of confidence \citep{zhao2021calibrate,xiong2024can}, where stated probabilities are systematically too extreme or too mild, and \emph{sycophancy} \citep{perez2023discovering,sharma2024towards}, where an agent's report is pulled toward the previous speaker's report regardless of its own evidence. We analyze a single round of such distortion: the previous agent is calibrated and sends the ideal $p_{k-1}$, while the round-$k$ agent outputs a biased report,
\begin{equation}\label{eq:biased-report}
    \hat p_k = \tfrac{1}{2}+c\big(p_k-\tfrac{1}{2}\big)+\lambda\big(p_{k-1}-\tfrac{1}{2}\big),
\end{equation}
with $(c,\lambda)$ chosen so that $\hat p_k\in[0,1]$.  The parameter $c>0$ controls the confidence scale ($c=1$ is calibrated, $c>1$ overconfident, $c<1$ underconfident), and $\lambda\in\mathbb R$ controls the strength of sycophantic anchoring to the previous report. This illustrates the broader problem: raw probability reports need not be reliable objects for the next agent to update on.
To recover calibrated messages, we can learn the conditional label distribution from the communicated probabilities themselves.

\begin{theorem}[Post-hoc calibration recovers the prediction guarantee]
\label{thm:cross-cal-recovery}
In the chained parity model, suppose the previous message is the ideal posterior $p_{k-1}$. For any $c>0$ and $\lambda\in\mathbb R$ so that $\hat p_k\in[0,1]$, let $Z=(\hat p_k,p_{k-1})$. If the predictor class contains the conditional label distribution, then training a predictor $h$ to predict $a^\star$ from $Z$ using a strictly proper loss yields, at the population optimum, the one-step post-hoc forecast
\[
    h^\star(\hat p_k,p_{k-1}) := \Pr(a^\star=1\mid \hat p_k,p_{k-1})
\]
which recovers the posterior: $h^\star(\hat p_k,p_{k-1})=p_k$. Moreover, the learned forecast is conversation-calibrated. Applied inductively, with each previous calibrated message equal to the corresponding ideal posterior, this recovers the guarantee of Theorem~\ref{thm:prediction-aggregates} for any $(c,\lambda)$.
\end{theorem}
\textit{Proof sketch.} A strictly proper loss is minimized at the true conditional label distribution, so the population optimum satisfies $h^\star(Z)=\Pr(a^\star=1\mid Z)$. Since $c>0$, observing $(\hat p_k,p_{k-1})$ in~\eqref{eq:biased-report} determines $p_k$; the unrevealed tail still has conditional mean zero given these messages. Therefore $\Pr(a^\star=1\mid \hat p_k,p_{k-1})=p_k$. Finally, because $Z$ includes the previous forecast, the tower property gives $\Pr(a^\star=1\mid h^\star(Z),p_{k-1})=h^\star(Z),$
so the learned forecast is cross-calibrated with respect to $p_{k-1}$. Full details are in Appendix~\ref{app:calibration}.

The affine form in~\eqref{eq:biased-report} is only a convenient example. The same conclusion holds for any distortion $\hat p_k=F(p_k,p_{k-1})$ such that, after fixing $p_{k-1}$, the map $p_k\mapsto F(p_k,p_{k-1})$ is invertible on the possible round-$k$ posteriors.
The gap from calibration is already visible in the first two rounds: if a receiver treats a biased report as a valid posterior and projects it to the nearest possible message, the first bit can be decoded correctly while the second bit may be flipped by confidence or anchoring bias. Appendix~\ref{app:calibration} formalizes this.

\textbf{A post-hoc calibration algorithm.}
The above observation suggests a natural training procedure for collaborative agents that communicate probabilistic predictions. In each round, the active agent produces a raw forecast from its own observation and the previous message. A post-hoc map then recalibrates this forecast using labelled transcripts and the previous message as inputs. This yields an iterative procedure in which agents alternately improve their predictive accuracy while aligning their communicated probabilities through conversation calibration. We summarize this procedure in Algorithm~\ref{alg:cross_calibrated_conversation}.


\begin{algorithm}[H]
\caption{post-hoc calibration for probability communication}
\label{alg:cross_calibrated_conversation}
\begin{algorithmic}[1]
\Require Data $\{(x_i^A,x_i^B,a_i^\star)\}_{i=1}^n$, rounds $K$, base protocols $\tilde \pi_k$ producing raw forecasts, strictly proper loss $\ell$
\State Initialize $\hat p_{0,i}=\emptyset$ for all $i=1,\ldots,n$
\For{$k=1,\ldots,K$}
    \State Let $s=\sigma(k)$ be the round-$k$ speaker
    \State For each example $i=1,...,n$, run the base protocol to obtain raw forecasts $\tilde p_{k,i}
        =
        \tilde \pi_k(x_i^s,\hat p_{k-1,i}).$
    \State Fit a post-hoc calibration map using only the raw forecast and previous calibrated forecast:
    \[
        g_k
        \in
        \arg\min_{g\in\mathcal G_k}
        \sum_{i=1}^n
        \ell\!\left(g(\tilde p_{k,i},\hat p_{k-1,i}),a_i^\star\right).
    \]
    \State For each example $i=1,...,n$, set the calibrated round-$k$ message $\hat p_{k,i}=g_k(\tilde p_{k,i},\hat p_{k-1,i}).$
\EndFor
\State \Return Calibrators $g_1,\ldots,g_K$
\end{algorithmic}
\end{algorithm}

\section{Maze experiments with learned and pretrained agents}\label{sec:maze}

We next move beyond chained parities to test these communication principles on a more natural decision problem: collaborative maze-solving. The task, proposed by \citet{davidson2025collaboration} to test collaborative capabilities of language models, preserves the same core structure: agents have partial information and must aggregate information over multiple rounds before choosing an action. However, the underlying object of uncertainty is now much more complex: agents must form and update beliefs about the correct next move by reasoning over partially observed paths. This setting therefore lets us test whether the lessons from the chained parities example extend to richer collaborative problems.

\begin{figure}[t]
    \centering
    \begin{minipage}[t]{0.41\textwidth}
        \centering
        \includegraphics[width=\linewidth,trim={0cm 0cm 0cm 0cm},clip]{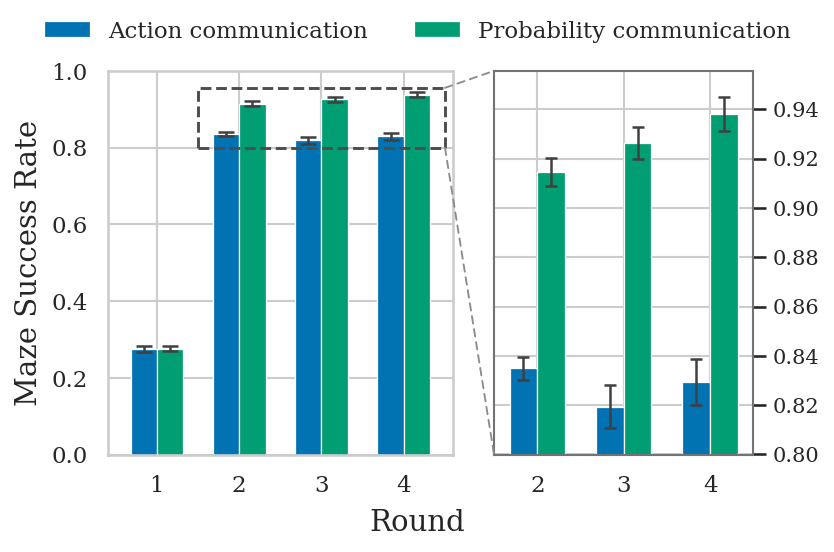}
        \caption{\textbf{(R1)}
        Action performance stagnates; communicating probabilities increases performance gradually.
        }
        \label{fig:maze-success}
    \end{minipage}
    \hfill
    \begin{minipage}[t]{0.55\textwidth}
        \centering
        \includegraphics[width=\linewidth,trim={0cm 0cm 0cm 0cm},clip]{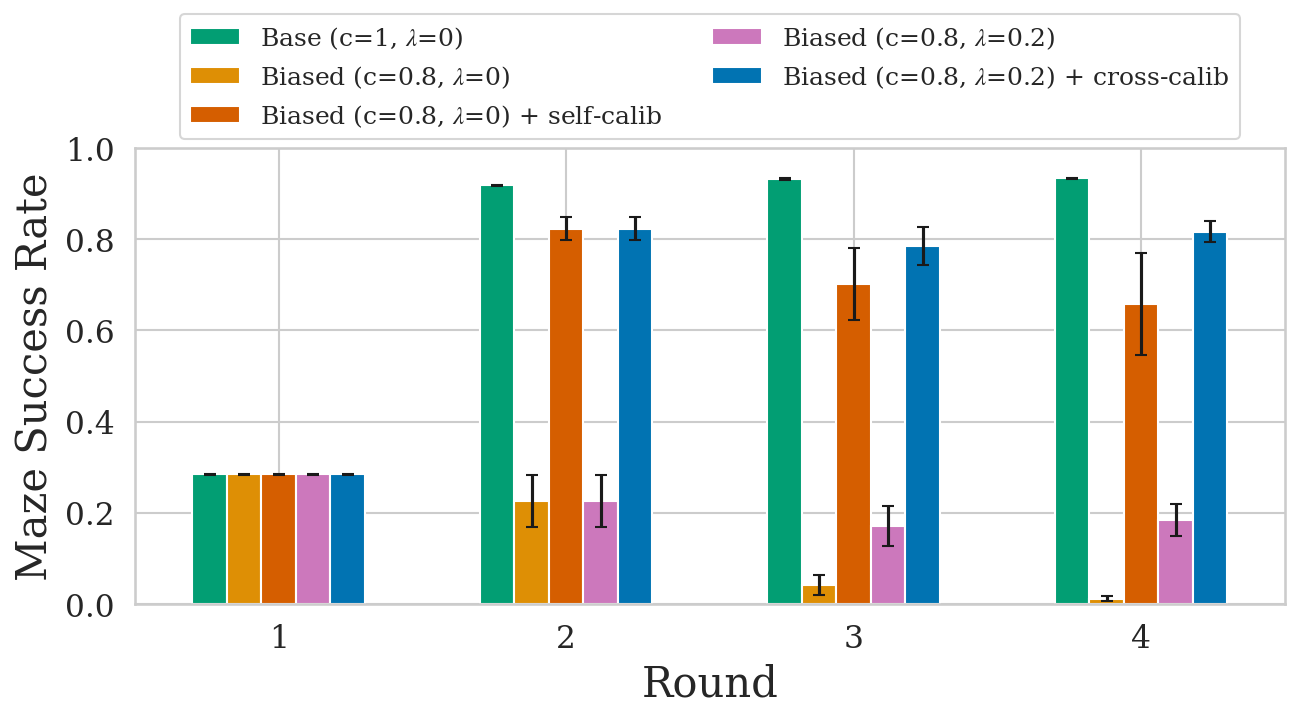}
        \caption{ \textbf{(R2)}
        Global miscalibration and sycophantic anchoring cause performance to drop over rounds, while post-hoc recalibration partially recovers performance under self-miscalibration and cross-calibration recovers nearly all performance when both biases are present.}
        \label{fig:posthoc-maze-success}
    \end{minipage}
\end{figure}

\noindent\textbf{Maze task.} The maze-solving task consists of a dataset of $N\times N$ mazes, each with a fixed start state, fixed end state, and cells denoting either a path or a wall, see Figure \ref{fig:communication-protocols}. 
Mazes are sampled according to a wall-density parameter, which controls the expected fraction of cells that are walls. In our collaborative setup, each agent observes only a partial map of the maze with half of the cells hidden. Together, the two agents' maps recover the full maze. Our experimental setup follows the primary setting studied by \citet{davidson2025collaboration}. We use $N=6$, a wall density of 30\%. \rebuttal{The start state is in the top left corner and the end state in the bottom right corner, which creates a directional bias where some actions are more likely than others.} Unlike \citet{davidson2025collaboration}, we restrict attention to mazes with a single correct solution for ease of evaluation. An example maze (and the corresponding two views) are presented in Figure~\ref{fig:maze-example}.
\subsection{Controlled transformer experiments}
\label{sec:maze-transformers}
We first study the problem in a controlled setup, where, as in Section~\ref{sec:chained-parities}, we train a sequence of transformers that communicate across rounds. 

\noindent\textbf{Transformer setup.} We train small transformers on a supervised next-move prediction task. Each example is represented as a partial-view maze string with the @ symbol denoting the current position, along with any previously communicated message. The target label is the next move along the solution path. Full experimental details are in Appendix~\ref{app:exp_details}. We report maze success rate, defined as the fraction of mazes for which the model predicts every move in the solution path correctly.

\noindent\textbf{(R1) Probability messages improve controlled maze-solving.}  Figure~\ref{fig:maze-success} shows that trained models under probability communication indeed outperform trained models under action communication. Moreover, it shows that over four rounds of communication, maze success rate improves consistently when agents communicate probabilities; this is not the case for action communication. In Appendix \ref{app:hard-examples} we give examples of mazes where communicating probabilities over multiple rounds converge to the correct answer.

\noindent\textbf{(R2) Calibration repairs biased transformer messages.}
Next, we study the two biases that conversation calibration is designed to mitigate when agents communicate probabilities: global miscalibration and sycophantic anchoring. Rather than communicating the transformer probabilities directly, we first distort them according to Equation~\ref{eq:biased-report}, and then optionally apply a post-hoc debiasing protocol. The debiasing model is a single-layer network trained with cross-entropy loss against the true label. We compare self-calibration, which uses only the current transformer's probabilities, with conversation calibration, which uses both the current and previous round's probabilities.

The unbiased probability baseline improves monotonically over four rounds; see Figure~\ref{fig:posthoc-maze-success}. In contrast, global miscalibration alone and global miscalibration combined with anchoring cause performance to collapse and deteriorate over rounds. Self-calibration recovers a substantial fraction of the lost success under global miscalibration, but its imperfect corrections still accumulate through the conversation. When both biases are present, conversation calibration recovers nearly all of the baseline performance. Thus, cross-calibrating against the previous round's probabilities can recover approximately unbiased messages and substantially improve downstream maze success.

\subsection{Pretrained LLM experiments}
\label{sec:llm}

We now test if the same principle holds for large-scale, pretrained language models on the maze task.

\noindent\textbf{LLM setup.} We use the maze task described in Section \ref{sec:maze} and evaluate four models of varying strengths: GPT-4.1-mini, GPT-5, GPT-5.4-nano, and GPT-5.4-mini. We allow models to generate reasoning (up to 2048 tokens). We consider three communication protocols: actions, verbalized probabilities, or sampled probabilities. In the \textit{action} exchange protocol, each agent communicates the action chosen in its response. In the \textit{verbalized probability} exchange protocol, each agent communicates the probability vector reported in its response. In the \textit{sampled probability} exchange protocol, we sample $M$ independent responses from the model, and the message communicated is the empirical distribution over the answers. For both the verbalized and sampled probability exchange protocol, we measure the accuracy of the answer assigned the highest probability value. Full prompts are in Appendix~\ref{app:llm_exp_details}. \rebuttal{All results are over 30 fixed mazes with 10 rollouts per maze. For maze success rates, we report means and $90\%$ bootstrapped confidence intervals over all rollouts, computed as a pooled variance across mazes. We compute conversation calibration losses separately for each of the 10 rollouts, again pooling moves from all 30 mazes, and report the mean and  $90\%$ bootstrapped confidence intervals across rollouts.}

\begin{figure}[!t]
    \centering

    \begin{subfigure}[t]{\textwidth}
        \centering

        \begin{subfigure}[t]{0.6\textwidth}
            \centering
            \vspace{0pt}
            \includegraphics[width=\linewidth]{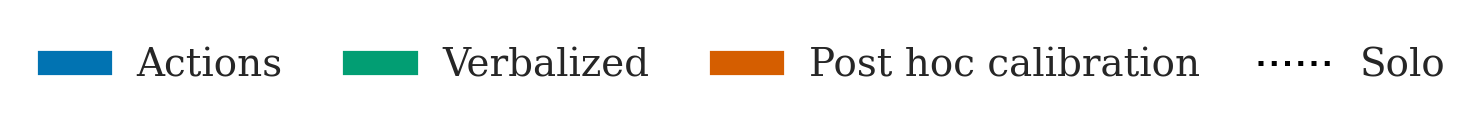}
        \end{subfigure}

        \par\vspace{-0.8em}

        \begin{subfigure}[t]{0.48\textwidth}
            \centering
            \vspace{0pt}
            \includegraphics[width=\linewidth]{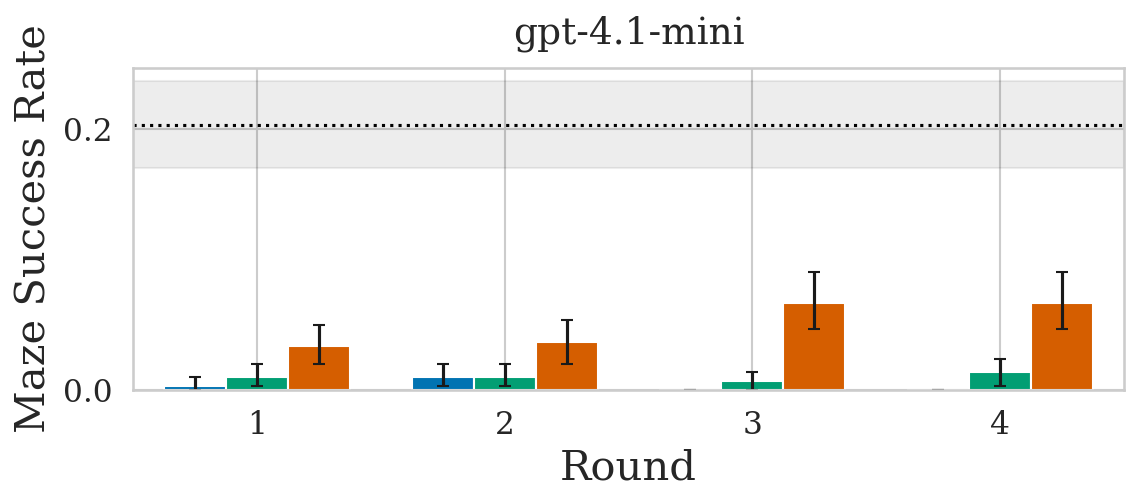}
        \end{subfigure}%
        \hspace{0.02\textwidth}%
        \begin{subfigure}[t]{0.45\textwidth}
            \centering
            \vspace{0pt}
            \includegraphics[width=\linewidth, trim={1cm 0 0 0}, clip]{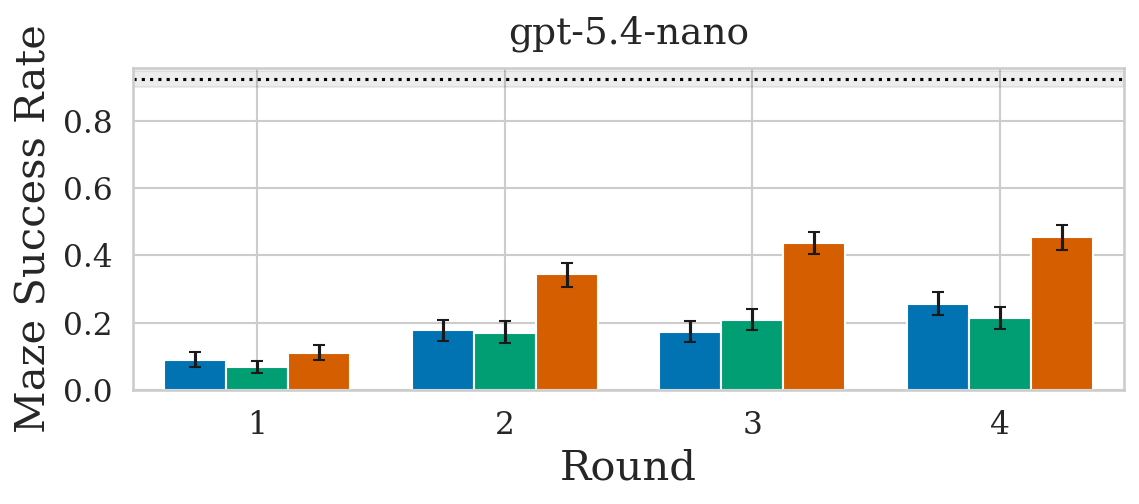}
        \end{subfigure}

        \ifarxiv
        \par\vspace{-1.5em}
        \fi

        \begin{subfigure}[t]{0.48\textwidth}
            \centering
            \vspace{0pt}
            \includegraphics[width=\linewidth]{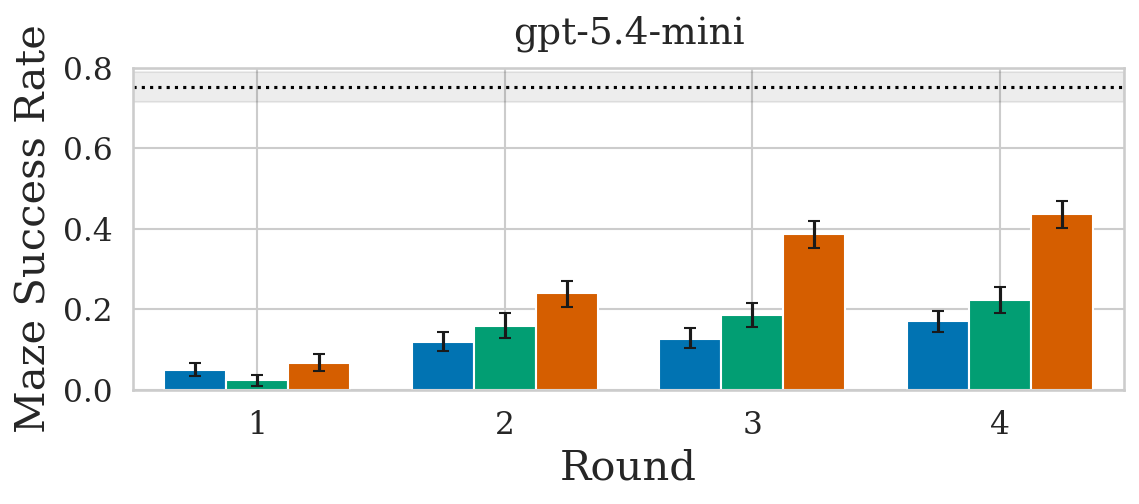}
        \end{subfigure}%
        \hspace{0.02\textwidth}%
        \begin{subfigure}[t]{0.45\textwidth}
            \centering
            \vspace{0pt}
            \includegraphics[width=\linewidth, trim={1cm 0 0 0}, clip]{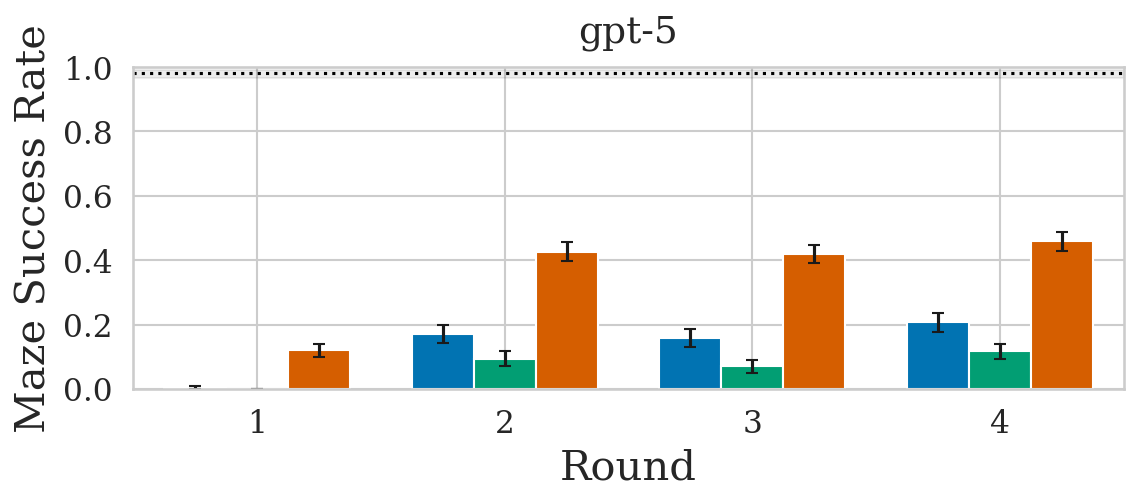}
        \end{subfigure}

        \caption{Maze success rate.}
        \label{fig:compare-comm-panel}
    \end{subfigure}


    \begin{subfigure}[t]{\textwidth}
        \centering

        \begin{subfigure}[t]{0.38\textwidth}
            \centering
            \vspace{0pt}
            \includegraphics[width=\linewidth]{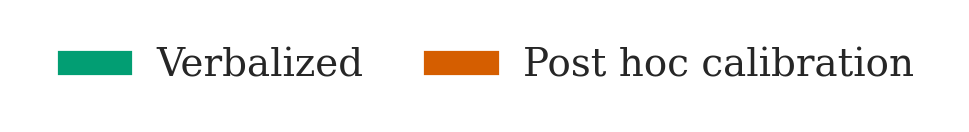}
        \end{subfigure}

        \par\vspace{-0.8em}

        \begin{subfigure}[t]{0.38\textwidth}
            \centering
            \vspace{0pt}
            \includegraphics[width=\linewidth, trim={0cm 1cm 0 0}, clip]{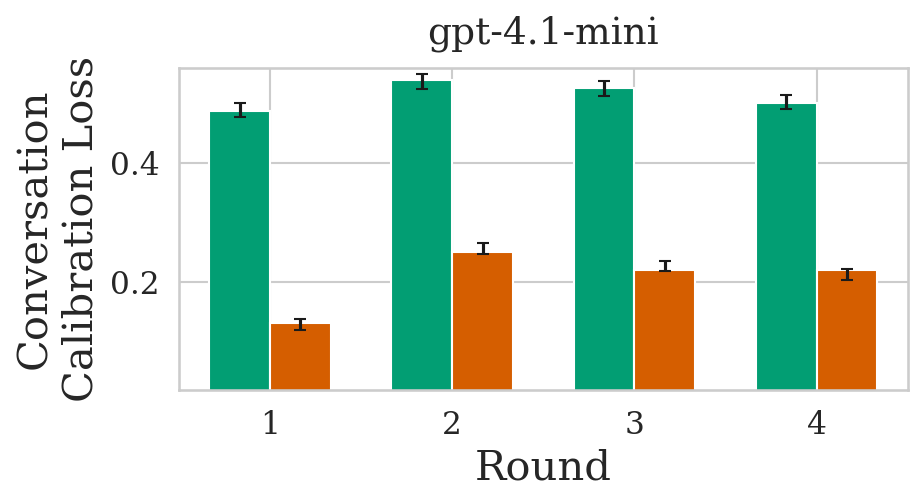}
        \end{subfigure}%
        \hspace{0.1\textwidth}%
        \begin{subfigure}[t]{0.34\textwidth}
            \centering
            \vspace{0pt}
            \includegraphics[width=\linewidth, trim={1.8cm 0.1cm 0 0}, clip]{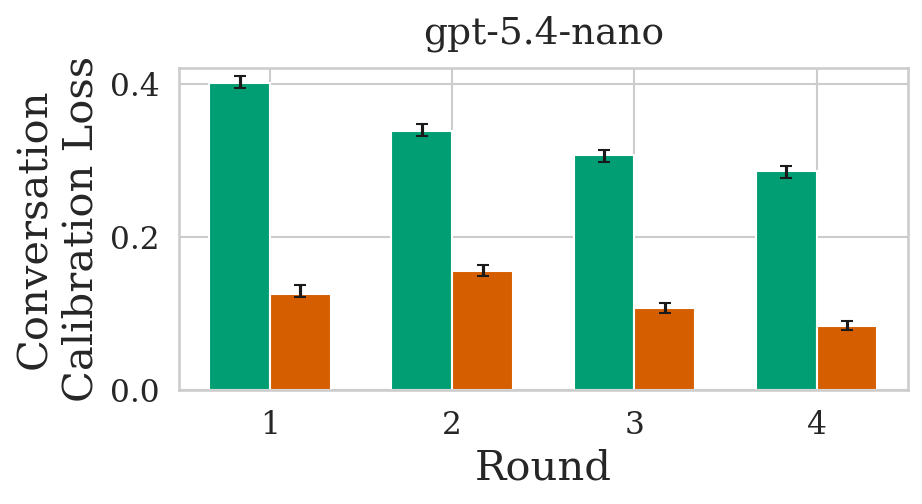}
        \end{subfigure}

        \par\vspace{-1.4em}

        \begin{subfigure}[t]{0.38\textwidth}
            \centering
            \vspace{0pt}
            \includegraphics[width=\linewidth, trim={0cm 0 0 0}, clip]{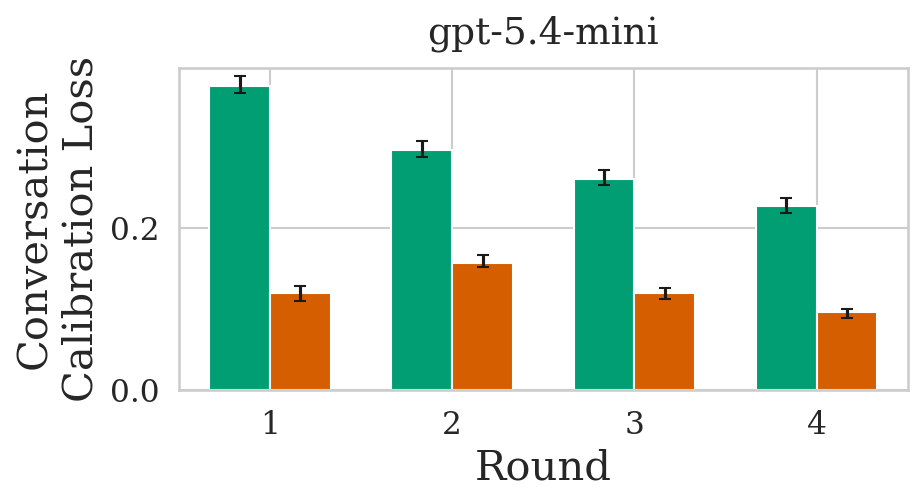}
        \end{subfigure}%
        \hspace{0.1\textwidth}%
        \begin{subfigure}[t]{0.34\textwidth}
            \centering
            \vspace{0pt}
            \includegraphics[width=\linewidth, trim={1.8cm 0.1cm 0 0}, clip]{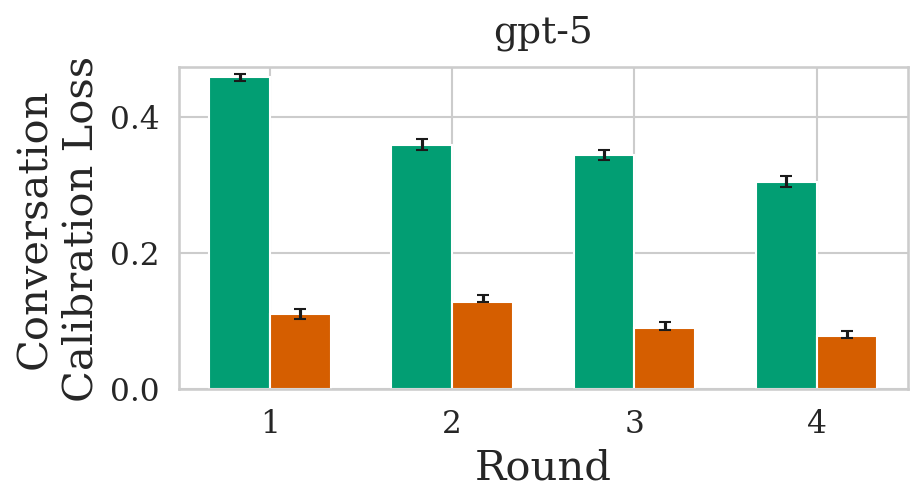}
        \end{subfigure}

        \caption{Conversation calibration loss.}
        \label{fig:cal-panel}
    \end{subfigure}

    \par\vspace{-0.5em}

    \caption{\textbf{(R3)} We train one-hidden-layer MLP calibrators according to Algorithm \ref{alg:cross_calibrated_conversation} on a set of 100 mazes and report outcomes on a validation set of 30 mazes with 10 rollouts each. We report the success rate (top) as well as the calibration loss (bottom). For success rate, we also include the full-information solo baseline. Bars that do not appear indicate zero success. Post-hoc calibration of verbalized probabilities increases success and lowers conversation calibration loss. Details of conversation calibration loss computation is in Appendix \ref{app:llm_exp_details}.}
    \label{fig:compare-comm-cal}
\end{figure}

\noindent\textbf{(R3) Calibration makes verbalized LLM probabilities useful.} In our first experiment, we highlight the value of conversation calibration for verbalized probabilities. Recent work has shown that the probabilities verbalized by LLMs often fail to faithfully reflect the models' underlying uncertainty~\citep{xiong2024can, kapoor2024large}. To illustrate the consequences of this miscalibration in a collaborative setting, we directly compare action-based and probability-based communication using prompted LLMs. We then apply post-hoc conversation calibration to the verbalized probabilities to test whether the lost information can be recovered through simple debiasing.

In Figure~\ref{fig:compare-comm-panel}, we report the success rates of solving the mazes across multiple models and communication modes. We find that verbalized probability exchange performs similarly to, or worse than, action exchange for GPT-5.4-nano, GPT-5.4-mini, and GPT-5, and yields only a small improvement for GPT-4.1-mini. Single verbal reports of uncertainty therefore do not appear to faithfully represent the model's uncertainty in a form that improves collaboration. Applying Algorithm~\ref{alg:cross_calibrated_conversation} to verbalized probabilities consistently improves performance across all models. In particular, mean maze success rates at least double after four rounds of communication relative to the base verbalized probability protocol. These gains suggest that communicated uncertainty can support collaboration when the exchanged probabilities are first adjusted to account for model-specific miscalibration.

To illustrate that this improved performance is in fact correlated with (conversation) calibration, we show the corresponding calibration losses of the verbalized probabilities before and after post-processing in Figure~\ref{fig:cal-panel}. In the first round, we report self calibration loss and in consecutive rounds we report conversation calibration loss. For all four models, post-hoc calibration leads to a significant decrease in calibration loss, supporting our hypothesis that conversation calibration is in fact a driving factor in improving communication in our setting.

\rebuttal{
We note that final performance is far below what can be achieved by training a next-move predictor on the maze data, as shown in the previous section. This gap can be attributed to three factors. First, the LLMs were not explicitly trained on mazes like ours, and so their prior knowledge this task is limited. Second, the calibration procedure does not help the model re-interpret information: it is trained only on the communicated probabilities and therefore does not introduce new knowledge specific to each maze that is absent from what is communicated. 
If the verbalized probabilities carry little information, calibration cannot alleviate that. 
Third, the LLM must interpret probabilities in-context, and it may not reason reliably about uncertainty information. Whereas prior work has highlighted the benefits of calibration in a single-agent setting~\citep{zhao2021calibrate}, our results point to the need for improved uncertainty calibration in multi-agent contexts.
}

\begin{figure}[t]
    \centering

    \begin{subfigure}{0.69\textwidth}
        \centering
        \includegraphics[width=\linewidth]{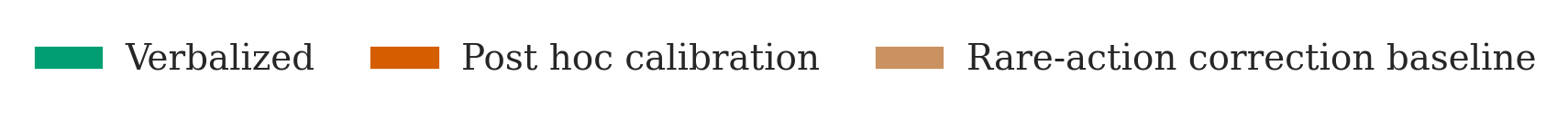}
    \end{subfigure}

    \begin{subfigure}{0.48\textwidth}
        \centering
        \includegraphics[width=\linewidth]{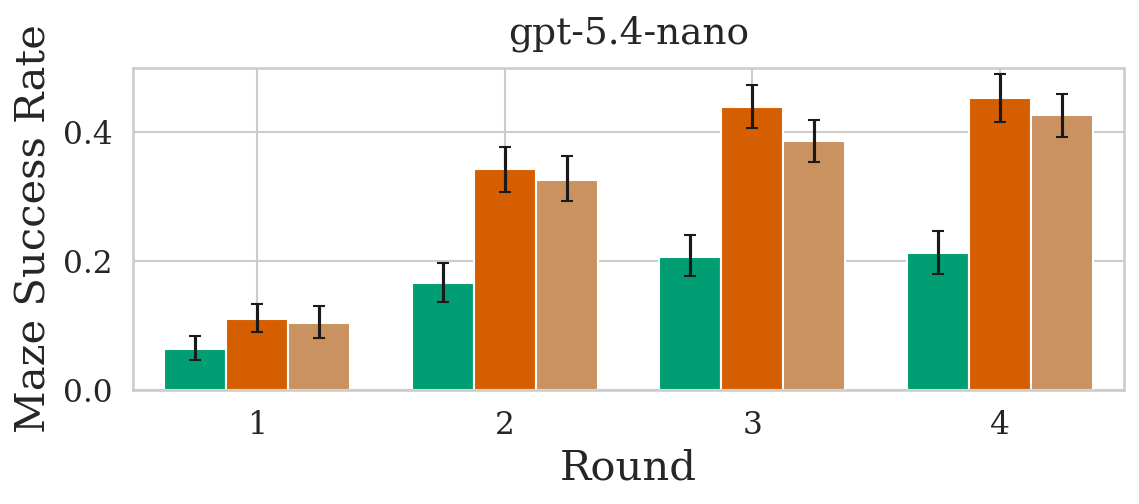}
        \label{fig:grid1}
    \end{subfigure}
    \hspace{0.01\textwidth}%
    \begin{subfigure}{0.45\textwidth}
        \centering
        \includegraphics[width=\linewidth, trim={1cm 0 0 0 }, clip]{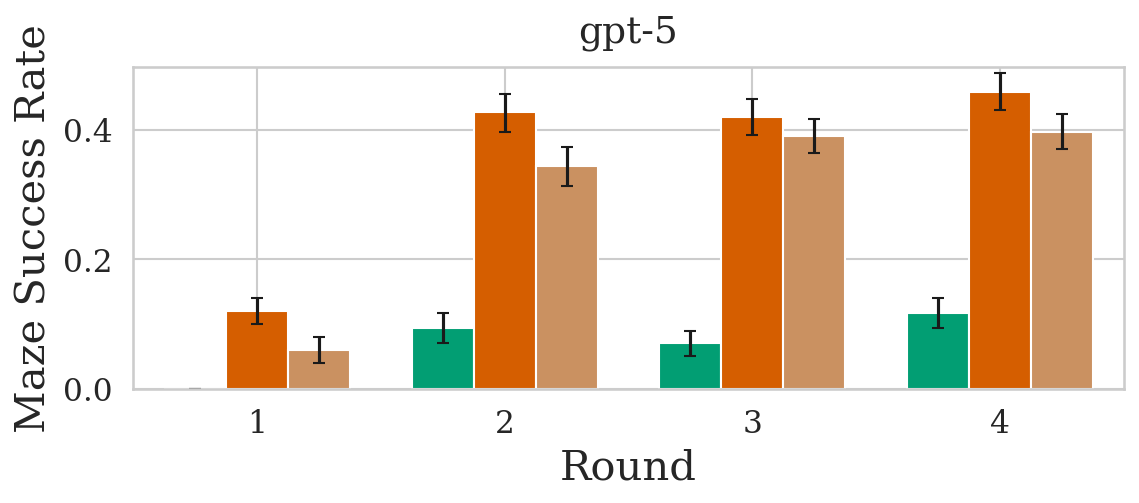}
        \label{fig:grid2}
    \end{subfigure}

    \caption{\rebuttal{\textbf{(R4)} Post-hoc calibration can learn corrections beyond the marginal action frequencies. On GPT-5.4-nano, the calibrator performs similarly to a naive baseline that replaces verbalized predictions favoring rare actions with the marginal action distribution. On GPT-5, the calibrator outperforms this baseline after multiple rounds of communication.}}
    \label{fig:naive-baseline}
\end{figure}

\rebuttal{\noindent\textbf{(R4) What does the post-hoc calibrator learn?} Next, to better understand where the calibrator's gains come from, we investigate what the post-hoc calibration procedure learns, and in particular whether it learns anything beyond the marginal action frequencies. Because our mazes start from the top left, the correct action is down or right 97\% of the time (with roughly equal frequencies on down and right), while left or up only 3\% of the time. Since the calibrator is trained on the distribution of correct actions, it could exploit this imbalance by simply learning a correction that suppresses predictions on the rare actions. We test this hypothesis with a naive rare-action correction baseline. Whenever the model’s verbalized probability vector assigns the highest probability to up or left, we replace it with the vector assigning probability 0.5 to each of down and right; otherwise, we leave the prediction unchanged. Note that this baseline uses only the marginal action frequencies and does not depend on the models' reported probabilities beyond the argmax.}

\rebuttal{Figure \ref{fig:naive-baseline} shows the results. 
For GPT-5.4-nano, the naive baseline performs similarly to the calibrator, indicating that much of the calibrator's gains can be explained by correcting the mode bias. 
For GPT-5, the calibrator slightly outperforms the naive baseline. 
This suggests that the calibration procedure learns corrections that cannot simply be reduced to suppressing rare actions and that LLMs can sometimes leverage nuanced probabilities to improve downstream performance.
Further work on analyzing and improving how LLMs output and interpret calibrated probabilities offers the opportunity to unlock the full benefits of conversation calibration.
}


\noindent\textbf{(R5) Sampled probabilities only partially elicit model uncertainty.}
Having established that verbalized probabilities are unreliable, we next compare against a stronger elicitation baseline: empirical probabilities from multiple sampled rollouts. This protocol estimates the model's action distribution directly from repeated responses, but is more expensive than a single verbalized probability report. We compare sampled against verbalized-probability exchange as well as post-hoc calibration.

\begin{figure}[t]
    \centering

    \begin{subfigure}{0.5\textwidth}
        \centering
        \includegraphics[width=\linewidth]{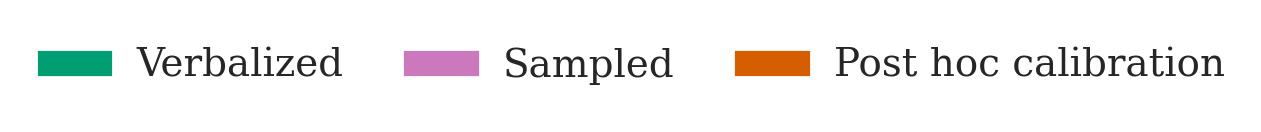}
    \end{subfigure}

    \begin{subfigure}{0.48\textwidth}
        \centering
        \includegraphics[width=\linewidth]{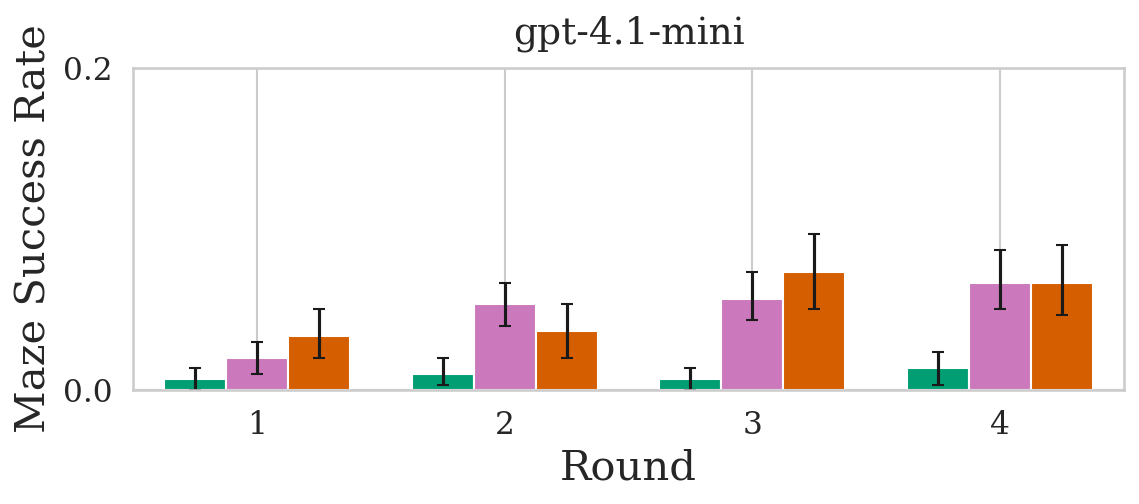}
        \label{fig:grid1}
    \end{subfigure}
    \hspace{0.01\textwidth}%
    \begin{subfigure}{0.45\textwidth}
        \centering
        \includegraphics[width=\linewidth, trim={1cm 0 0 0 }, clip]{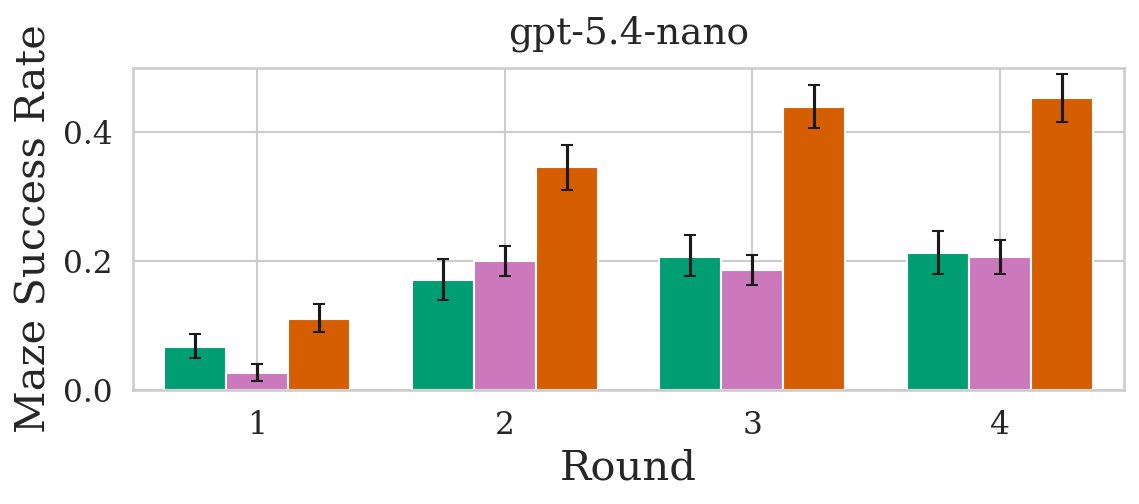}
        \label{fig:grid2}
    \end{subfigure}

    \caption{\textbf{(R5)} Our post-hoc calibration procedure matches or outperforms sampled probabilities constructed from $M=10$ sampled responses.}
    \label{fig:sampled}
\end{figure}

\rebuttal{Figure~\ref{fig:sampled} shows that on GPT-4.1-mini, sampled probability exchange improves over verbalized probability exchange and is within variance of our post-hoc calibration approach. On GPT-5.4-nano, sampled probability exchange performs similarly to verbalized probability exchange. Here, post-hoc calibration consistently achieves higher mean success rates. These results suggest that this simple post-hoc procedure can compete with a sampling-based uncertainty estimate at a fraction of the cost, since it only requires one rollout.}

\begin{figure}[t]
    \centering

    \begin{subfigure}{0.8\textwidth}
        \centering
        \includegraphics[width=\linewidth]{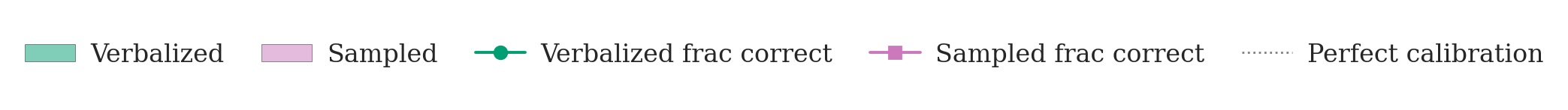}
    \end{subfigure}

    \begin{subfigure}{0.45\textwidth}
        \centering
        \includegraphics[width=\linewidth]{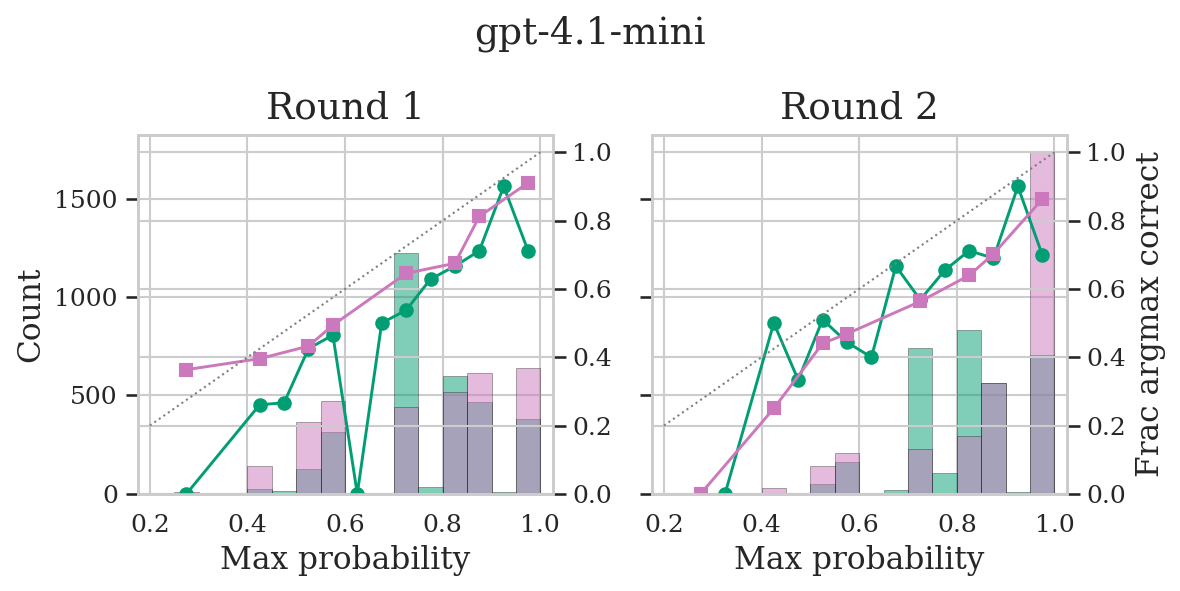}
        \label{fig:grid1}
    \end{subfigure}
    \hfill
    \begin{subfigure}{0.45\textwidth}
        \centering
        \includegraphics[width=\linewidth]{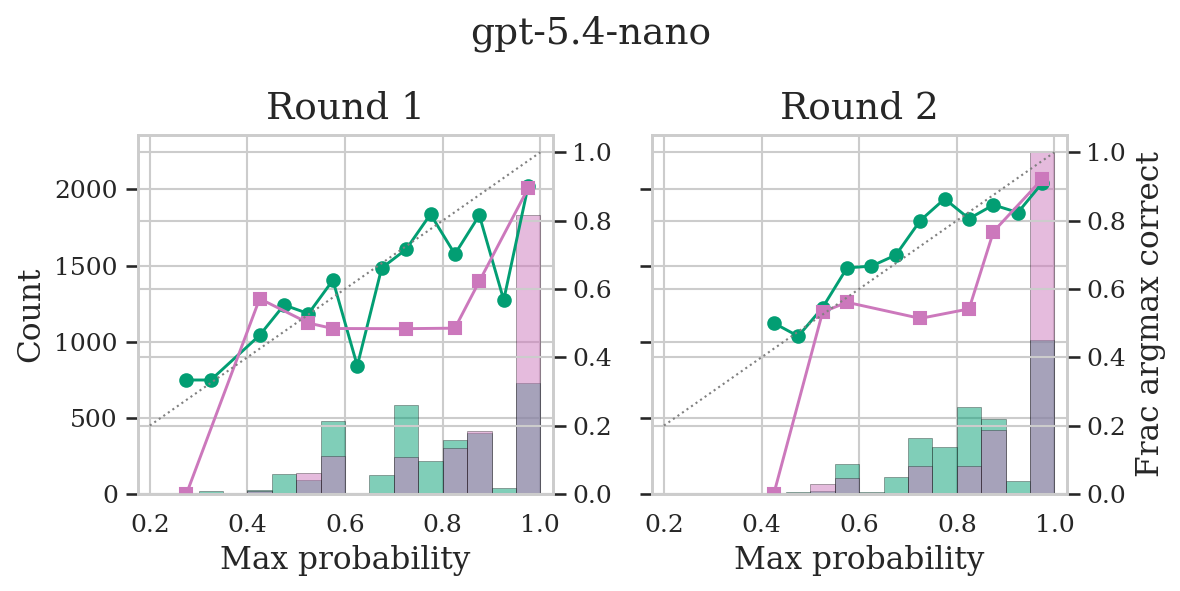}
        \label{fig:grid2}
    \end{subfigure}

    \caption{\textbf{(R6)} Histogram of maximum assigned probabilities and corresponding argmax-action accuracy under verbalized and sampled probability exchange. For GPT-4.1-mini, sampled probabilities are better calibrated than verbalized probabilities and lead to stronger downstream performance. For GPT-5.4-nano, sampled probabilities are more overconfident, while verbalized probabilities are comparatively better calibrated.}
    \label{fig:histogram}
\end{figure}

\noindent\textbf{(R6) Calibration explains when sampled probabilities help.} GPT-4.1-mini behaves as expected: it performs poorly with verbalized probabilities but improves with sampled probabilities. For GPT-5.4-nano, however, sampled probabilities are on par with verbalized probabilities. This might be surprising at first glance, since we would expect the sampled probability measure to perform better according to our hypothesis. To understand why this might happen, we dig deeper into the distribution of both the verbalized and sampled probabilities in Figure~\ref{fig:histogram}. \rebuttal{We plot a histogram of maximum probability values under the verbalized and sampled probability exchange protocols, together with the accuracy of the corresponding argmax action.
For GPT-4.1-mini, the verbalized and sampled top-choice confidences are similarly miscalibrated across most bins in both rounds $1$ and $2$. However, in round $2$, sampled probabilities more often assign a maximum probability of $1$, and these predictions are reasonably well-calibrated. By contrast, verbalized probabilities continue to be overconfident. As a result, sampled probabilities achieve higher accuracy in round $2$.
For GPT-5.4-nano, we observe a similar pattern in the $1.0$ probability bin, but sampled probabilities do not outperform verbalized probabilities overall. This can be attributed to the fact that, outside the $1.0$ probability bin, the sampled probabilities are highly overconfident.}
This contrast further highlights the close relationship between calibration and downstream performance. The behavior of both models is consistent with broader trends in LLMs. As model capability improves, LLMs appear to become better at expressing uncertainty through verbalized probabilities~\citep{xiong2024can}, which may help explain GPT-5.4-nano’s improved verbal calibration. One possible explanation is that GPT-5.4-nano’s sampled rollouts reflect stronger commitment due to more alignment training such as RLHF. This is consistent with prior work linking RLHF-style alignment to sharpened output distributions and increased overconfidence~\citep{leng2025taming}.

\begin{figure}[t]
    \centering

    \begin{subfigure}{0.4\textwidth}
        \centering
        \includegraphics[width=\linewidth]{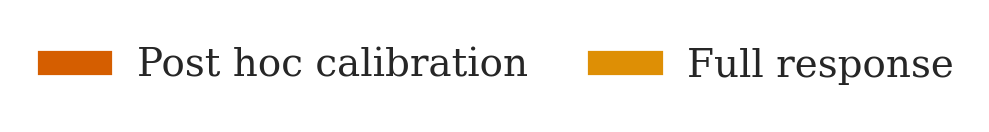}
    \end{subfigure}

    \begin{subfigure}{0.38\textwidth}
        \centering
        \includegraphics[width=\linewidth]{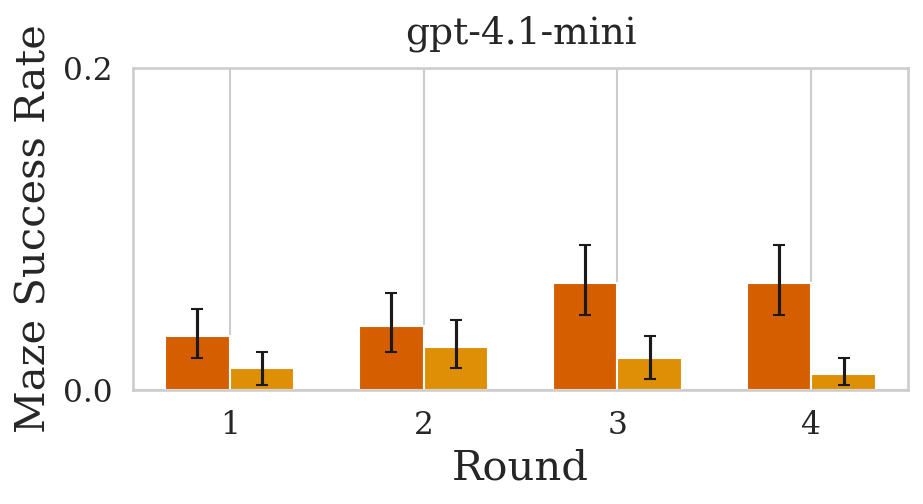}
    \end{subfigure}
    \hspace{0.01\textwidth}%
    \begin{subfigure}{0.35\textwidth}
        \centering
        \includegraphics[width=\linewidth, trim={1cm 0 0 0 }, clip]{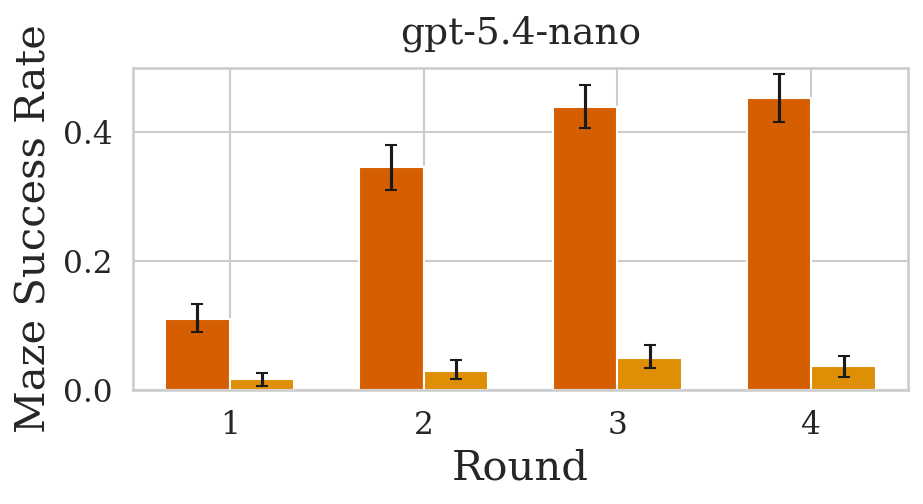}
    \end{subfigure}

    \caption{\textbf{(R7)} Communicating more information does not always help. Solely communicating probabilities produced by our post-hoc calibration procedure leads to higher success than communicating the full response.}
    \label{fig:full-response}
\end{figure}

\noindent\textbf{(R7) Communicating more information does not necessarily improve collaboration.} Finally, using GPT-4.1-mini and GPT-5.4-nano, we compare against an unconstrained communication baseline, in which each agent receives the previous agent's full response, including any generated reasoning. This is the setup studied by \citet{davidson2025collaboration}. For GPT-4.1-mini, we use the same prompt as in our other experiments, which asks the model for its uncertainty; therefore the communicated full response includes both the model's reasoning \textit{and} verbalized probabilities. For GPT-5.4-nano, we slightly modify the prompt to also explicitly have the model tell us what to communicate to the other agent, as otherwise it's reasoning is hidden inside the API. We provide the modified prompt in Appendix~\ref{app:exp_details}. The results are found in Figure~\ref{fig:full-response}.

For both models, we see that unrestricted communication performs poorly as its success rate across all rounds is close to zero. 
This failure is unlikely to reflect a basic inability to parse or reason about any maze structure the communicated message may contain, since GPT-5.4-nano solves more than $85\%$ of the mazes when given the full maze directly (Figure~\ref{fig:compare-comm-cal}).
Rather, this experiment demonstrates that simply providing the next agent with \textit{more} information does not necessarily improve collaboration, and probability exchange can be more effective between agents than unconstrained natural-language communication, especially when long or noisy contexts make full-response communication difficult to use.

\ifarxiv
\subsection*{Acknowledgements}
The authors gratefully acknowledge support from a Google PhD Fellowship, the Simons Collaboration on Algorithmic Fairness, an Amazon Research Award, and the NSF EnCore Tripods Institute. SG was additionally supported by NSF CAREER \#2543725. MH was partially supported by DARPA grant \#HR00112420305, the U.S. National Science Foundation under Cooperative Agreement No. 2433450. 
Any opinions, findings and conclusions or recommendations expressed in this material are those of the author(s) and do not necessarily reflect the views of the U.S. army or the U.S. National Science Foundation.
MH is also a Novo Nordisk Fellow supported by the MIT-Novo Nordisk
Artificial Intelligence Postdoctoral Fellows Program. 
\fi

\bibliographystyle{plainnat}
\bibliography{NeurIPS26/updated_updated_plots}

\begin{thebibliography}{46}
\providecommand{\natexlab}[1]{#1}
\providecommand{\url}[1]{\texttt{#1}}
\expandafter\ifx\csname urlstyle\endcsname\relax
  \providecommand{\doi}[1]{doi: #1}\else
  \providecommand{\doi}{doi: \begingroup \urlstyle{rm}\Url}\fi

\bibitem[Aaronson(2005)]{aaronson2004complexity}
Scott Aaronson.
\newblock The complexity of agreement.
\newblock In \emph{Proceedings of the thirty-seventh annual ACM symposium on Theory of computing}, pages 634--643, 2005.

\bibitem[Aumann(1976)]{aumann1976}
Robert~J. Aumann.
\newblock {Agreeing to Disagree}.
\newblock \emph{The Annals of Statistics}, 4\penalty0 (6):\penalty0 1236 -- 1239, 1976.
\newblock \doi{10.1214/aos/1176343654}.
\newblock URL \url{https://doi.org/10.1214/aos/1176343654}.

\bibitem[Aumann and Hart(2006)]{aumann2006agreeing}
Robert~J Aumann and Sergiu Hart.
\newblock Agreeing on decisions.
\newblock \emph{Unpublished manuscript, The Einstein Institute of Mathematics, Jerusalem, Israel, http://math. huji. ac. il/\~{} hart/papers/agree. pdf (Last accessed: 17/04/2015)}, 2006.

\bibitem[Cemri et~al.(2025)Cemri, Pan, Yang, Agrawal, Chopra, Tiwari, Keutzer, Parameswaran, Klein, Ramchandran, Zaharia, Gonzalez, and Stoica]{cemri2026why}
Mert Cemri, Melissa~Z Pan, Shuyi Yang, Lakshya~A Agrawal, Bhavya Chopra, Rishabh Tiwari, Kurt Keutzer, Aditya Parameswaran, Dan Klein, Kannan Ramchandran, Matei Zaharia, Joseph~E. Gonzalez, and Ion Stoica.
\newblock Why do multi-agent {LLM} systems fail?
\newblock In \emph{The Thirty-ninth Annual Conference on Neural Information Processing Systems Datasets and Benchmarks Track}, 2025.

\bibitem[Chan et~al.(2024)Chan, Chen, Su, Yu, Xue, Zhang, Fu, and Liu]{chan2024chateval}
Chi-Min Chan, Weize Chen, Yusheng Su, Jianxuan Yu, Wei Xue, Shanghang Zhang, Jie Fu, and Zhiyuan Liu.
\newblock {ChatEval}: Towards better {LLM}-based evaluators through multi-agent debate.
\newblock In \emph{The Twelfth International Conference on Learning Representations}, 2024.
\newblock URL \url{https://openreview.net/forum?id=FQepisCUWu}.

\bibitem[Chen et~al.(2024)Chen, Saha, and Bansal]{chen-etal-2024-reconcile}
Justin Chen, Swarnadeep Saha, and Mohit Bansal.
\newblock {R}e{C}oncile: Round-table conference improves reasoning via consensus among diverse {LLM}s.
\newblock In Lun-Wei Ku, Andre Martins, and Vivek Srikumar, editors, \emph{Proceedings of the 62nd Annual Meeting of the Association for Computational Linguistics}, pages 7066--7085, Bangkok, Thailand, August 2024. Association for Computational Linguistics.
\newblock \doi{10.18653/v1/2024.acl-long.381}.
\newblock URL \url{https://aclanthology.org/2024.acl-long.381/}.

\bibitem[Collina et~al.(2025)Collina, Goel, Gupta, and Roth]{collina2025tractable}
Natalie Collina, Surbhi Goel, Varun Gupta, and Aaron Roth.
\newblock Tractable agreement protocols.
\newblock In \emph{Proceedings of the 57th Annual ACM Symposium on Theory of Computing}, pages 1532--1543, 2025.

\bibitem[Collina et~al.(2026)Collina, Globus-Harris, Goel, Gupta, Roth, and Shi]{collina2026collaborative}
Natalie Collina, Ira Globus-Harris, Surbhi Goel, Varun Gupta, Aaron Roth, and Mirah Shi.
\newblock Collaborative prediction: Tractable information aggregation via agreement.
\newblock In \emph{Proceedings of the 2026 Annual ACM-SIAM Symposium on Discrete Algorithms (SODA)}, pages 4712--4798. SIAM, 2026.

\bibitem[Dalvi et~al.(2021)Dalvi, Jansen, Tafjord, Xie, Smith, Pipatanangkura, and Clark]{entailmenttree}
Bhavana Dalvi, Peter Jansen, Oyvind Tafjord, Zhengnan Xie, Hannah Smith, Leighanna Pipatanangkura, and Peter Clark.
\newblock Explaining answers with entailment trees.
\newblock In Marie-Francine Moens, Xuanjing Huang, Lucia Specia, and Scott Wen-tau Yih, editors, \emph{Proceedings of the 2021 Conference on Empirical Methods in Natural Language Processing}, pages 7358--7370, Online and Punta Cana, Dominican Republic, November 2021. Association for Computational Linguistics.
\newblock \doi{10.18653/v1/2021.emnlp-main.585}.
\newblock URL \url{https://aclanthology.org/2021.emnlp-main.585/}.

\bibitem[Das et~al.(2019)Das, Gervet, Romoff, Batra, Parikh, Rabbat, and Pineau]{pmlr-v97-das19a}
Abhishek Das, Th{\'e}ophile Gervet, Joshua Romoff, Dhruv Batra, Devi Parikh, Mike Rabbat, and Joelle Pineau.
\newblock {T}ar{MAC}: Targeted multi-agent communication.
\newblock In Kamalika Chaudhuri and Ruslan Salakhutdinov, editors, \emph{Proceedings of the 36th International Conference on Machine Learning}, volume~97 of \emph{Proceedings of Machine Learning Research}, pages 1538--1546. PMLR, 09--15 Jun 2019.
\newblock URL \url{https://proceedings.mlr.press/v97/das19a.html}.

\bibitem[Davidson et~al.(2025)Davidson, Fourney, Amershi, West, Horvitz, and Kamar]{davidson2025collaboration}
Tim~R Davidson, Adam Fourney, Saleema Amershi, Robert West, Eric Horvitz, and Ece Kamar.
\newblock The collaboration gap.
\newblock \emph{arXiv preprint arXiv:2511.02687}, 2025.

\bibitem[Dawid(1982)]{dawid1982well}
A.~P. Dawid.
\newblock The well-calibrated bayesian.
\newblock \emph{Journal of the American Statistical Association}, 77\penalty0 (379):\penalty0 605--610, 1982.
\newblock \doi{10.1080/01621459.1982.10477856}.

\bibitem[Desai and Durrett(2020)]{desai-durrett-2020-calibration}
Shrey Desai and Greg Durrett.
\newblock Calibration of pre-trained transformers.
\newblock In Bonnie Webber, Trevor Cohn, Yulan He, and Yang Liu, editors, \emph{Proceedings of the 2020 Conference on Empirical Methods in Natural Language Processing}, pages 295--302, Online, November 2020. Association for Computational Linguistics.
\newblock \doi{10.18653/v1/2020.emnlp-main.21}.
\newblock URL \url{https://aclanthology.org/2020.emnlp-main.21/}.

\bibitem[Du et~al.(2024)Du, Li, Torralba, Tenenbaum, and Mordatch]{pmlr-v235-du24e}
Yilun Du, Shuang Li, Antonio Torralba, Joshua~B. Tenenbaum, and Igor Mordatch.
\newblock Improving factuality and reasoning in language models through multiagent debate.
\newblock In Ruslan Salakhutdinov, Zico Kolter, Katherine Heller, Adrian Weller, Nuria Oliver, Jonathan Scarlett, and Felix Berkenkamp, editors, \emph{Proceedings of the 41st International Conference on Machine Learning}, volume 235 of \emph{Proceedings of Machine Learning Research}, pages 11733--11763. PMLR, 21--27 Jul 2024.
\newblock URL \url{https://proceedings.mlr.press/v235/du24e.html}.

\bibitem[Eisenstein et~al.(2026)Eisenstein, Huot, Fisch, Berant, and Lapata]{eisenstein2026mtpingevalevaluatingmultiturncollaboration}
Jacob Eisenstein, Fantine Huot, Adam Fisch, Jonathan Berant, and Mirella Lapata.
\newblock Mt-pingeval: Evaluating multi-turn collaboration with private information games, 2026.
\newblock URL \url{https://arxiv.org/abs/2602.24188}.

\bibitem[Farquhar et~al.(2024)Farquhar, Kossen, Kuhn, and Gal]{farquhar2024detecting}
Sebastian Farquhar, Jannik Kossen, Lorenz Kuhn, and Yarin Gal.
\newblock Detecting hallucinations in large language models using semantic entropy.
\newblock \emph{Nature}, 630\penalty0 (8017):\penalty0 625--630, 2024.
\newblock \doi{10.1038/s41586-024-07421-0}.
\newblock URL \url{https://www.nature.com/articles/s41586-024-07421-0}.

\bibitem[Foerster et~al.(2016)Foerster, Assael, de~Freitas, and Whiteson]{foerster2016learning}
Jakob~N. Foerster, Ioannis~Alexandros Assael, Nando de~Freitas, and Shimon Whiteson.
\newblock Learning to communicate with deep multi-agent reinforcement learning.
\newblock In \emph{Advances in Neural Information Processing Systems}, volume~29, 2016.
\newblock URL \url{https://proceedings.neurips.cc/paper/2016/hash/c7635bfd99248a2cdef8249ef7bfbef4-Abstract.html}.

\bibitem[Foster and Vohra(1998)]{foster1998asymptotic}
Dean~P. Foster and Rakesh~V. Vohra.
\newblock Asymptotic calibration.
\newblock \emph{Biometrika}, 85\penalty0 (2):\penalty0 379--390, 1998.
\newblock \doi{10.1093/biomet/85.2.379}.

\bibitem[Frongillo et~al.(2023)Frongillo, Neyman, and Waggoner]{frongillo2023agreement}
Rafael Frongillo, Eric Neyman, and Bo~Waggoner.
\newblock Agreement implies accuracy for substitutable signals.
\newblock In \emph{Proceedings of the 24th ACM Conference on Economics and Computation}, pages 702--733, 2023.

\bibitem[Geanakoplos and Polemarchakis(1982)]{geanakoplos1982we}
John~D Geanakoplos and Heraklis~M Polemarchakis.
\newblock We can't disagree forever.
\newblock \emph{Journal of Economic theory}, 28\penalty0 (1):\penalty0 192--200, 1982.

\bibitem[Guo et~al.(2017)Guo, Pleiss, Sun, and Weinberger]{pmlr-v70-guo17a}
Chuan Guo, Geoff Pleiss, Yu~Sun, and Kilian~Q. Weinberger.
\newblock On calibration of modern neural networks.
\newblock In Doina Precup and Yee~Whye Teh, editors, \emph{Proceedings of the 34th International Conference on Machine Learning}, volume~70 of \emph{Proceedings of Machine Learning Research}, pages 1321--1330. PMLR, 06--11 Aug 2017.
\newblock URL \url{https://proceedings.mlr.press/v70/guo17a.html}.

\bibitem[Guo et~al.(2024)Guo, Chen, Wang, Chang, Pei, Chawla, Wiest, and Zhang]{ijcai2024p890}
Taicheng Guo, Xiuying Chen, Yaqi Wang, Ruidi Chang, Shichao Pei, Nitesh~V. Chawla, Olaf Wiest, and Xiangliang Zhang.
\newblock Large language model based multi-agents: A survey of progress and challenges.
\newblock In Kate Larson, editor, \emph{Proceedings of the Thirty-Third International Joint Conference on Artificial Intelligence, {IJCAI-24}}, pages 8048--8057. International Joint Conferences on Artificial Intelligence Organization, August 2024.
\newblock \doi{10.24963/ijcai.2024/890}.
\newblock URL \url{https://doi.org/10.24963/ijcai.2024/890}.
\newblock Survey Track.

\bibitem[Jiang et~al.(2021)Jiang, Araki, Ding, and Neubig]{jiang-etal-2021-know}
Zhengbao Jiang, Jun Araki, Haibo Ding, and Graham Neubig.
\newblock How can we know when language models know? on the calibration of language models for question answering.
\newblock \emph{Transactions of the Association for Computational Linguistics}, 9:\penalty0 962--977, 2021.
\newblock \doi{10.1162/tacl_a_00407}.
\newblock URL \url{https://aclanthology.org/2021.tacl-1.57/}.

\bibitem[Kadavath et~al.(2022)Kadavath, Conerly, Askell, Henighan, Drain, Perez, Schiefer, Hatfield-Dodds, DasSarma, Tran-Johnson, Johnston, El-Showk, Jones, Elhage, Hume, Chen, Bai, Bowman, Fort, Ganguli, Hernandez, Jacobson, Kernion, Kravec, Lovitt, Ndousse, Olsson, Ringer, Amodei, Brown, Clark, Joseph, Mann, McCandlish, Olah, and Kaplan]{kadavath2022language}
Saurav Kadavath, Tom Conerly, Amanda Askell, Tom Henighan, Dawn Drain, Ethan Perez, Nicholas Schiefer, Zac Hatfield-Dodds, Nova DasSarma, Eli Tran-Johnson, Scott Johnston, Sheer El-Showk, Andy Jones, Nelson Elhage, Tristan Hume, Anna Chen, Yuntao Bai, Sam Bowman, Stanislav Fort, Deep Ganguli, Danny Hernandez, Josh Jacobson, Jackson Kernion, Shauna Kravec, Liane Lovitt, Kamal Ndousse, Catherine Olsson, Sam Ringer, Dario Amodei, Tom Brown, Jack Clark, Nicholas Joseph, Ben Mann, Sam McCandlish, Chris Olah, and Jared Kaplan.
\newblock Language models (mostly) know what they know, 2022.
\newblock URL \url{https://arxiv.org/abs/2207.05221}.

\bibitem[Kapoor et~al.(2024)Kapoor, Gruver, Roberts, Collins, Pal, Bhatt, Weller, Dooley, Goldblum, and Wilson]{kapoor2024large}
Sanyam Kapoor, Nate Gruver, Manley Roberts, Katherine~M. Collins, Arka Pal, Umang Bhatt, Adrian Weller, Samuel Dooley, Micah Goldblum, and Andrew~Gordon Wilson.
\newblock Large language models must be taught to know what they don{\textquoteright}t know.
\newblock In \emph{The Thirty-eighth Annual Conference on Neural Information Processing Systems}, 2024.

\bibitem[Karpathy(2022)]{karpathy2022nanogpt}
Andrej Karpathy.
\newblock \text{NanoGPT}.
\newblock \url{https://github.com/karpathy/nanoGPT}, 2022.

\bibitem[Kearns et~al.(2026)Kearns, Roth, and Ryu]{kearns2026networked}
Michael Kearns, Aaron Roth, and Emily Ryu.
\newblock Networked information aggregation via machine learning.
\newblock In \emph{Proceedings of the 2026 Annual ACM-SIAM Symposium on Discrete Algorithms (SODA)}, pages 4799--4845. SIAM, 2026.

\bibitem[Kuhn et~al.(2023)Kuhn, Gal, and Farquhar]{kuhn2023semantic}
Lorenz Kuhn, Yarin Gal, and Sebastian Farquhar.
\newblock Semantic uncertainty: Linguistic invariances for uncertainty estimation in natural language generation.
\newblock In \emph{The Eleventh International Conference on Learning Representations}, 2023.
\newblock URL \url{https://openreview.net/forum?id=VD-AYtP0dve}.

\bibitem[Leng et~al.(2025)Leng, Huang, Zhu, and Huang]{leng2025taming}
Jixuan Leng, Chengsong Huang, Banghua Zhu, and Jiaxin Huang.
\newblock Taming overconfidence in {LLM}s: Reward calibration in {RLHF}.
\newblock In \emph{The Thirteenth International Conference on Learning Representations}, 2025.
\newblock URL \url{https://openreview.net/forum?id=l0tg0jzsdL}.

\bibitem[Li et~al.(2023)Li, Chong, Stepputtis, Campbell, Hughes, Lewis, and Sycara]{li2023theory}
Huao Li, Yu~Quan Chong, Simon Stepputtis, Joseph Campbell, Dana Hughes, Charles~Michael Lewis, and Katia~P. Sycara.
\newblock Theory of mind for multi-agent collaboration via large language models.
\newblock In \emph{The 2023 Conference on Empirical Methods in Natural Language Processing}, 2023.

\bibitem[Lin et~al.(2022)Lin, Hilton, and Evans]{lin2022teaching}
Stephanie Lin, Jacob Hilton, and Owain Evans.
\newblock Teaching models to express their uncertainty in words.
\newblock \emph{Transactions on Machine Learning Research}, 2022.
\newblock URL \url{https://openreview.net/forum?id=8s8K2UZGTZ}.

\bibitem[Manakul et~al.(2023)Manakul, Liusie, and Gales]{manakul-etal-2023-selfcheckgpt}
Potsawee Manakul, Adian Liusie, and Mark J.~F. Gales.
\newblock {S}elf{C}heck{GPT}: Zero-resource black-box hallucination detection for generative large language models.
\newblock In Houda Bouamor, Juan Pino, and Kalika Bali, editors, \emph{Proceedings of the 2023 Conference on Empirical Methods in Natural Language Processing}, pages 9004--9017, Singapore, December 2023. Association for Computational Linguistics.
\newblock \doi{10.18653/v1/2023.emnlp-main.557}.
\newblock URL \url{https://aclanthology.org/2023.emnlp-main.557/}.

\bibitem[Mielke et~al.(2022)Mielke, Szlam, Dinan, and Boureau]{mielke-etal-2022-reducing}
Sabrina~J. Mielke, Arthur Szlam, Emily Dinan, and Y-Lan Boureau.
\newblock Reducing conversational agents' overconfidence through linguistic calibration.
\newblock \emph{Transactions of the Association for Computational Linguistics}, 10:\penalty0 857--872, 2022.
\newblock \doi{10.1162/tacl_a_00494}.
\newblock URL \url{https://aclanthology.org/2022.tacl-1.50/}.

\bibitem[Perez et~al.(2023)Perez, Ringer, Lukosiute, Nguyen, Chen, Heiner, Pettit, Olsson, Kundu, Kadavath, Jones, Chen, Mann, Israel, Seethor, McKinnon, Olah, Yan, Amodei, Amodei, Drain, Li, Tran-Johnson, Khundadze, Kernion, Landis, Kerr, Mueller, Hyun, Landau, Ndousse, Goldberg, Lovitt, Lucas, Sellitto, Zhang, Kingsland, Elhage, Joseph, Mercado, DasSarma, Rausch, Larson, McCandlish, Johnston, Kravec, El~Showk, Lanham, Telleen-Lawton, Brown, Henighan, Hume, Bai, Hatfield-Dodds, Clark, Bowman, Askell, Grosse, Hernandez, Ganguli, Hubinger, Schiefer, and Kaplan]{perez2023discovering}
Ethan Perez, Sam Ringer, Kamile Lukosiute, Karina Nguyen, Edwin Chen, Scott Heiner, Craig Pettit, Catherine Olsson, Sandipan Kundu, Saurav Kadavath, Andy Jones, Anna Chen, Benjamin Mann, Brian Israel, Bryan Seethor, Cameron McKinnon, Christopher Olah, Da~Yan, Daniela Amodei, Dario Amodei, Dawn Drain, Dustin Li, Eli Tran-Johnson, Guro Khundadze, Jackson Kernion, James Landis, Jamie Kerr, Jared Mueller, Jeeyoon Hyun, Joshua Landau, Kamal Ndousse, Landon Goldberg, Liane Lovitt, Martin Lucas, Michael Sellitto, Miranda Zhang, Neerav Kingsland, Nelson Elhage, Nicholas Joseph, Noemi Mercado, Nova DasSarma, Oliver Rausch, Robin Larson, Sam McCandlish, Scott Johnston, Shauna Kravec, Sheer El~Showk, Tamera Lanham, Timothy Telleen-Lawton, Tom Brown, Tom Henighan, Tristan Hume, Yuntao Bai, Zac Hatfield-Dodds, Jack Clark, Samuel~R. Bowman, Amanda Askell, Roger Grosse, Danny Hernandez, Deep Ganguli, Evan Hubinger, Nicholas Schiefer, and Jared Kaplan.
\newblock Discovering language model behaviors with model-written evaluations.
\newblock In Anna Rogers, Jordan Boyd-Graber, and Naoaki Okazaki, editors, \emph{Findings of the Association for Computational Linguistics: ACL 2023}, pages 13387--13434, Toronto, Canada, July 2023. Association for Computational Linguistics.
\newblock \doi{10.18653/v1/2023.findings-acl.847}.
\newblock URL \url{https://aclanthology.org/2023.findings-acl.847/}.

\bibitem[Sharma et~al.(2024)Sharma, Tong, Korbak, Duvenaud, Askell, Bowman, DURMUS, Hatfield-Dodds, Johnston, Kravec, Maxwell, McCandlish, Ndousse, Rausch, Schiefer, Yan, Zhang, and Perez]{sharma2024towards}
Mrinank Sharma, Meg Tong, Tomasz Korbak, David Duvenaud, Amanda Askell, Samuel~R. Bowman, Esin DURMUS, Zac Hatfield-Dodds, Scott~R Johnston, Shauna~M Kravec, Timothy Maxwell, Sam McCandlish, Kamal Ndousse, Oliver Rausch, Nicholas Schiefer, Da~Yan, Miranda Zhang, and Ethan Perez.
\newblock Towards understanding sycophancy in language models.
\newblock In \emph{The Twelfth International Conference on Learning Representations}, 2024.

\bibitem[Sinha et~al.(2019)Sinha, Sodhani, Dong, Pineau, and Hamilton]{clutrr}
Koustuv Sinha, Shagun Sodhani, Jin Dong, Joelle Pineau, and William~L. Hamilton.
\newblock {CLUTRR}: A diagnostic benchmark for inductive reasoning from text.
\newblock In \emph{Proceedings of the 2019 Conference on Empirical Methods in Natural Language Processing and the 9th International Joint Conference on Natural Language Processing (EMNLP-IJCNLP)}, 2019.

\bibitem[Sukhbaatar et~al.(2016)Sukhbaatar, szlam, and Fergus]{sukhbaatar2016learning}
Sainbayar Sukhbaatar, arthur szlam, and Rob Fergus.
\newblock Learning multiagent communication with backpropagation.
\newblock 29, 2016.

\bibitem[Tafjord et~al.(2021)Tafjord, Dalvi, and Clark]{proofwriter}
Oyvind Tafjord, Bhavana Dalvi, and Peter Clark.
\newblock {P}roof{W}riter: Generating implications, proofs, and abductive statements over natural language.
\newblock In \emph{Findings of the Association for Computational Linguistics: ACL-IJCNLP 2021}, 2021.

\bibitem[Tian et~al.(2023)Tian, Mitchell, Zhou, Sharma, Rafailov, Yao, Finn, and Manning]{tian-etal-2023-just}
Katherine Tian, Eric Mitchell, Allan Zhou, Archit Sharma, Rafael Rafailov, Huaxiu Yao, Chelsea Finn, and Christopher~D Manning.
\newblock Just ask for calibration: Strategies for eliciting calibrated confidence scores from language models fine-tuned with human feedback.
\newblock In \emph{Proceedings of the 2023 Conference on Empirical Methods in Natural Language Processing}, 2023.

\bibitem[Trivedi et~al.(2022)Trivedi, Balasubramanian, Khot, and Sabharwal]{musique}
Harsh Trivedi, Niranjan Balasubramanian, Tushar Khot, and Ashish Sabharwal.
\newblock Musique: Multihop questions via single-hop question composition.
\newblock \emph{Transactions of the Association for Computational Linguistics}, 10, 2022.

\bibitem[Wang et~al.(2023)Wang, Wei, Schuurmans, Le, Chi, Narang, Chowdhery, and Zhou]{wang2023selfconsistency}
Xuezhi Wang, Jason Wei, Dale Schuurmans, Quoc~V Le, Ed~H. Chi, Sharan Narang, Aakanksha Chowdhery, and Denny Zhou.
\newblock Self-consistency improves chain of thought reasoning in language models.
\newblock In \emph{The Eleventh International Conference on Learning Representations}, 2023.

\bibitem[Wei et~al.(2022)Wei, Wang, Schuurmans, Bosma, ichter, Xia, Chi, Le, and Zhou]{wei2022chain}
Jason Wei, Xuezhi Wang, Dale Schuurmans, Maarten Bosma, brian ichter, Fei Xia, Ed~Chi, Quoc~V Le, and Denny Zhou.
\newblock Chain-of-thought prompting elicits reasoning in large language models.
\newblock In \emph{Advances in Neural Information Processing Systems}, volume~35, 2022.

\bibitem[Wilie et~al.(2024)Wilie, Cahyawijaya, Ishii, He, and Fung]{wilie2024belief}
Bryan Wilie, Samuel Cahyawijaya, Etsuko Ishii, Junxian He, and Pascale Fung.
\newblock Belief revision: The adaptability of large language models reasoning.
\newblock In \emph{Proceedings of the 2024 Conference on Empirical Methods in Natural Language Processing}, 2024.

\bibitem[Xiong et~al.(2024)Xiong, Hu, Lu, Li, Fu, He, and Hooi]{xiong2024can}
Miao Xiong, Zhiyuan Hu, Xinyang Lu, Yifei Li, Jie Fu, Junxian He, and Bryan Hooi.
\newblock Can {LLM}s express their uncertainty? an empirical evaluation of confidence elicitation in {LLM}s.
\newblock In \emph{The Twelfth International Conference on Learning Representations}, 2024.

\bibitem[Yang et~al.(2018)Yang, Qi, Zhang, Bengio, Cohen, Salakhutdinov, and Manning]{hotpotqa}
Zhilin Yang, Peng Qi, Saizheng Zhang, Yoshua Bengio, William Cohen, Ruslan Salakhutdinov, and Christopher~D. Manning.
\newblock {H}otpot{QA}: A dataset for diverse, explainable multi-hop question answering.
\newblock In \emph{Proceedings of the 2018 Conference on Empirical Methods in Natural Language Processing}, 2018.

\bibitem[Zhao et~al.(2021)Zhao, Wallace, Feng, Klein, and Singh]{zhao2021calibrate}
Tony Zhao, Eric Wallace, Shi Feng, Dan Klein, and Sameer Singh.
\newblock Calibrate before use: Improving few-shot performance of language models.
\newblock In \emph{Proceedings of the 38th International Conference on Machine Learning}, 2021.

\end{thebibliography}


\appendix

\section{Additional Related Work}
\label{app:additional-related-work}

\paragraph{Calibration, scoring rules, and agreement-based aggregation.}
Our use of probabilistic messages is closely related to the classical literature on calibrated forecasting. \citet{dawid1982well} formalized the idea that a forecaster's stated probabilities should match empirical frequencies, while \citet{foster1998asymptotic} showed that calibrated forecasting can be achieved constructively in repeated prediction. The theoretical backbone of our work also connects to agreement and information aggregation. Classical agreement results show that Bayesian agents with common priors cannot sustain common-knowledge disagreement~\citep{aumann1976,aumann2006agreeing}, and that repeated exchange of beliefs can converge to agreement~\citep{geanakoplos1982we,aaronson2004complexity}. Recent work replaces ideal Bayesian assumptions with tractable calibration-style conditions~\citep{collina2025tractable,collina2026collaborative,kearns2026networked}. \citet{frongillo2023agreement} show that agreement need not imply accuracy without additional signal-structure assumptions. Our contribution is complementary: we study how these ideas behave in LLM-agent collaboration, and empirically compare action messages, raw probabilistic messages, sampled probabilities, and post-hoc conversation-calibrated probabilities.

\paragraph{Uncertainty and calibration in language models.}
A large body of work shows that neural-network and language-model probabilities are not automatically calibrated. \citet{pmlr-v70-guo17a} showed that modern neural networks can be substantially miscalibrated and that simple post-hoc temperature scaling can improve calibration. \citet{desai-durrett-2020-calibration} studied calibration in pretrained transformers such as BERT and RoBERTa, while \citet{jiang-etal-2021-know} evaluated when generative language models' confidence estimates align with correctness. \citet{zhao2021calibrate} showed that few-shot language-model predictions can be biased by prompt artifacts and proposed contextual calibration to mitigate these effects.

Recent work studies how to elicit uncertainty from large language models more directly. \cite{kadavath2022language} evaluate whether language models can predict their own correctness, finding that self-evaluation can be informative but is format- and scale-dependent. \cite{lin2022teaching} train models to express uncertainty in words, and~\cite{mielke-etal-2022-reducing} study whether conversational agents' linguistic expressions of confidence are calibrated. \cite{tian-etal-2023-just} compare confidence-elicitation strategies for RLHF-trained models, while \cite{xiong2024can} provide a broad empirical evaluation of verbalized, token-probability, and sampling-based confidence estimates. \citet{kapoor2024large} show that verbalized uncertainty in LLMs remains unreliable under prompting alone. These papers motivate our distinction between verbalized probabilities and probabilities that are calibrated for use by another agent in a conversation. In addition, work on sycophancy~\citep{perez2023discovering,sharma2024towards} motivates our biased-report model, where an agent's stated probability may be pulled toward a previous speaker's report rather than reflecting only its private evidence.

\paragraph{Sampling-based uncertainty and reasoning.}
Sampling multiple model outputs is another common way to estimate uncertainty. Chain-of-thought prompting~\citep{wei2022chain} and self-consistency decoding~\citep{wang2023selfconsistency} use multiple reasoning traces to improve single-agent reasoning by marginalizing over sampled answers. SelfCheckGPT~\citep{manakul-etal-2023-selfcheckgpt} uses consistency across sampled generations as a black-box signal for hallucination detection. \cite{kuhn2023semantic} and \cite{farquhar2024detecting} argue that uncertainty in language generation should often be measured over semantic equivalence classes rather than surface forms. Our sampled-probability protocol uses a related sampling signal, but for a different purpose: the empirical distribution over actions is communicated to another agent with complementary private information, rather than being used only to choose a single model's final answer.

\paragraph{Multi-agent LLM collaboration.}
Our maze experiments build on the collaborative maze-solving benchmark of \cite{davidson2025collaboration}, which demonstrates a ``collaboration gap'' between single-agent and multi-agent performance under partial observability. \citet{eisenstein2026mtpingevalevaluatingmultiturncollaboration} study stylistically similar tasks to evaluate multi-agent LLM collaboration in partially observable settings. Other work studies multi-agent LLM interaction as a general reasoning or evaluation scaffold. \cite{pmlr-v235-du24e} use multi-agent debate to improve factuality and reasoning, \cite{chen-etal-2024-reconcile} propose a round-table consensus procedure among diverse LLMs, and \cite{chan2024chateval} use multi-agent debate for LLM-based text evaluation. \cite{ijcai2024p890} survey the broader literature on LLM-based multi-agent systems. Unlike these works, our agents have different private observations, and our main question is not whether discussion helps in aggregate, but which low-bandwidth message format preserves decision-relevant information.

\paragraph{Learned communication in multi-agent reinforcement learning.}
There is also a large literature on learning communication protocols in cooperative multi-agent reinforcement learning. \cite{foerster2016learning} introduce RIAL and DIAL, showing that agents can learn discrete or differentiable communication protocols in partially observable cooperative tasks. \cite{sukhbaatar2016learning} propose CommNet, where agents learn continuous communication alongside their policies. \cite{pmlr-v97-das19a} introduce TarMAC, which learns targeted communication by deciding both what to communicate and whom to address. Our work differs from this line by focusing on an interpretable communication channel---beliefs over actions---and by identifying conversation calibration as the condition under which such messages can be reliably used by another agent.

\section{Warm-Up Example for Section \ref{sec:prob-vs-action}}
\label{ap:warmup}

We include a two-bit example showing that best-action communication can discard information that is preserved by probability communication. Let the state $y=(y_1,y_2)$ be drawn from $\{00,01,10,11\}$ with
\[
    \Pr(y=00)=\frac14-\epsilon,
    \qquad
    \Pr(y=01)=\Pr(y=10)=\Pr(y=11)=\frac14+\frac{\epsilon}{3},
\]
where $0<\epsilon<1/4$. Alice observes $y_1$, Bob observes $y_2$, and the correct action $a^\star=f^*(y)$ is
\[
\begin{array}{c|cccc}
y & 00 & 01 & 10 & 11 \\
\hline
f^*(y) & a_2 & a_1 & a_1 & a_1 .
\end{array}
\]
Thus $a_1$ is correct except in the single state $00$.

Suppose Alice speaks first and may send only her Bayes-optimal action. If $y_1=1$, she knows the state is either $10$ or $11$, so $a_1$ is certainly correct: $\Pr(a^\star=a_1\mid y_1=1)=1$. If $y_1=0$, Alice compares the two states $00$ and $01$, and still strictly prefers $a_1$:
\[
    \Pr(a^\star=a_1\mid y_1=0)
    =
    \Pr(y=01\mid y_1=0)
    =
    \left(\frac14+\frac{\epsilon}{3}\right)/\left(\frac12-\frac{2\epsilon}{3}\right)
    >
    \frac12.
\]
Therefore, she sends the same action, $a_1$, no matter which bit she observes: The action message carries no information about $y_1$. Bob faces the same situation symmetrically, so an action-only conversation cannot identify the exceptional state $00$; the agents play $a_1$, which is correct with prior probability $3/4+\epsilon$.

Under probability communication, Alice instead sends her posterior probability that $a_1$ is correct. This message equals $1$ when $y_1=1$ and equals $(\frac14+\frac{\epsilon}{3})/(\frac12-\frac{2\epsilon}{3})$ when $y_1=0$. These two values are distinct, so the probability message reveals her bit. Since Bob already observes $y_2$, he can identify the full state and report the correct action with probability one.

The lesson is that the separation is not between actions and probabilities as syntactic objects, but between messages that preserve distinct posterior states and messages that collapse all evidence states inducing the same best action. 


\section{Proofs for Section~\ref{sec:chained-parities}}
\label{app:chained-parity-proofs}

We start with a recurring observation used throughout. For shorthand, let
\begin{equation}
\label{notation:parity}
    u_k := \sum_{j\leq k}\alpha_j S_j,
    \qquad
    r_k := \sum_{j>k}\alpha_j S_j,
\end{equation}
with $u_0:=0$ and $r_0:=p^\star-1/2$, so that $p^\star=1/2+u_k+r_k$ for every $k$.

\begin{lemma}[Conditional vanishing]
\label{lem:vanishing}
Let $\mathcal G$ be a $\sigma$-algebra and suppose $b_i$ is independent of $\sigma(\mathcal G,\{b_\ell:\ell\neq i\})$ for some $i\in[K]$.  For any random variable $W$ that is a function of $\{b_\ell:\ell\neq i\}$, $\mathbb E[b_i W\mid\mathcal G]=\mathbb E[b_i]\,\mathbb E[W\mid\mathcal G]=0$.  In particular, if $J\subseteq[K]$ is nonempty and some $i\in J$ is independent of $\sigma(\mathcal G,\{b_\ell:\ell\neq i\})$, then $\mathbb E[\prod_{\ell\in J}b_\ell\mid\mathcal G]=0$.
\end{lemma}
\begin{proof}
Since $W$ is measurable with respect to $\sigma(\{b_\ell:\ell\neq i\})$, it is also measurable with respect to $\sigma(\mathcal G,\{b_\ell:\ell\neq i\})$. By the assumed independence,
\[
    \mathbb E[b_i\mid \sigma(\mathcal G,\{b_\ell:\ell\neq i\})]
    =
    \mathbb E[b_i]
    =
    0.
\]
Therefore
\[
    \mathbb E[b_iW\mid \mathcal G]
    =
    \mathbb E\!\left[
        W\,\mathbb E[b_i\mid \sigma(\mathcal G,\{b_\ell:\ell\neq i\})]
        \,\middle|\, \mathcal G
    \right]
    =
    0.
\]
For the product claim, write $\prod_{\ell\in J}b_\ell=b_iW$ with $W=\prod_{\ell\in J\setminus\{i\}}b_\ell$.
\end{proof}

\subsection{Proof of Theorem~\ref{thm:prediction-aggregates}}
\chainedParityPred*
\begin{proof}
Following the shorthand notation introduced above in Equation \ref{notation:parity}, we first prove that $p_k=1/2+\sum_{j\le k} \alpha_j S_j = 1/2 + u_k$ by induction on $k$.  For $k=1$, the previous message is empty; Alice knows $b_1$, and every later term $S_j$ with $j\geq 2$ contains Bob's unseen bit $b_2$, so Lemma~\ref{lem:vanishing} gives $\mathbb E[S_j\mid x_A]=0$ and therefore $p_1=1/2+\alpha_1 b_1=1/2+u_1$.

Inductively assume $p_{k-1}=1/2+u_{k-1}$.  By Assumption~\ref{ass:coefficients}, the round-$k$ speaker recovers the prefix $(S_1,\ldots,S_{k-1})$ from the scalar $p_{k-1}$; combined with its own bit $b_k$, the speaker also determines $S_k=S_{k-1}b_k$.  Every later $S_j$ with $j\geq k+1$ contains the next bit $b_{k+1}$, which has opposite parity from the speaker, so Lemma~\ref{lem:vanishing} forces its conditional mean to vanish. Hence
\begin{equation*}
    p_k=\mathbb E[p^\star\mid x_{\sigma(k)},p_{k-1}]
    =\frac{1}{2}+\textstyle\sum_{j\leq k}\alpha_j S_j
    =\frac{1}{2}+u_k.
\end{equation*}



We now turn to the claim about the excess risk induced by stopping conversation at round $k$ instead of achieving full information by stopping at round $K$. First, we note that we can rewrite the expected squared error of $p_k$ compared to the true action $a^*$ in terms of the difference between the risk induced by $p^*$, the weight of the Bernoulli coin that the true action is drawn from and $a^*$, and the difference between $p^*$ and $p_k$:
\begin{align*}
\E[(a^*-p_k)^2] &= \E[(a^*-p^* + p^* - p_k)^2] \\
&= \E[(a^*-p^*)^2] +2\E[(a^*-p^*)(p^*-p_k)] + \E[(p^*-p_k)^2]. \quad\quad\quad (*)
\end{align*}

Note that the cross term is equal to zero by the law of total expectation, conditioning on the value of the secret bits $b = (b_1, \ldots, b_K)$ held by Alice and Bob:
\begin{align*}
\E[(a^*-p^*)(p^*-p_k)] &= \E[(p^*-p_k)\E[a^*-p^* \vert b]] \\
&= \E[(p^* - p_k) \cdot 0] \\
&= 0.
\end{align*}
Thus, setting the cross term in Equation (*) to zero and rearranging terms, we find that
\[
\E[(a^*-p_k)^2] - \E[(a^*-p^*)^2] = \E[(p^*-p_k)^2].
\]
Note that 
\[
p^*-p_k = \sum_{j=1}^K \alpha_j S_j - \sum_{j=1}^k \alpha_j S_j = \sum_{j>k} \alpha_j S_j.
\]
Expanding, we find that
\begin{align*}
\E[(a^*-p_k)^2] - \E[(a^*-p^*)^2] &= \E\left[\left(\sum_{j>k} \alpha_j S_j\right)^2\right] \\
&= \E\left[\sum_{j>k} \alpha_j^2 S_j^2\right] + \E\left[\sum_{k < j < j' \le K} \alpha_j \alpha_{j'} S_j S_{j'}\right] \\
&= \sum_{j>k}\alpha_j^2 +\sum_{k < j < j' \le K} \alpha_j \alpha_{j'} \E[(b_1 \cdots b_j)(b_1 \cdots b_j \cdot b_{j+1} \cdots b_{j'}] \\
&= \sum_{j>k}\alpha_j^2 +\sum_{k < j < j' \le K} \alpha_j \alpha_{j'} \E[b_{j+1}\cdots b_{j'}] \\
&= \sum_{j>k} a_j^2 \quad\quad\quad\quad (**),
\end{align*}
since each of the bits $b_{j+1}, \ldots, b_{j'}$ are independent fair Bernoulli draws over $\{-1,1\}$ and hence their expected product is zero. 

For the final claim, about the improvement in risk in each round, we apply the above result, Equation (**), at both round $k-1$ and $k$:
\begin{align*}
\E[(a^*-p_{k-1})^2] - \E[(a^*-p_{k})^2] &= \big(\E[(a^*-p_{k-1})^2] - \E[(a^*-p^*)^2]\big) \\
&\quad\quad\quad - \big(\E[(a^*-p_{k})^2 - \E[(a^*-p^*)^2]\big) \\
&= \sum_{j>k-1} \alpha_j^2 - \sum_{j>k} \alpha_j^2 \\
&= \alpha_{k}^2.
\end{align*}
And, as a corollary to Equation (**) when $k=K$ and hence $\sum_{j>k} \alpha_j^2=0$, $\E[(a^*-p_K)^2] = \E[(a^*-p^*)^2]$.
\end{proof}

\begin{remark}[Full history vs.\ previous-message sufficiency]
\label{rem:full-history-bonus}
The distinct-partial-sums condition in Assumption~\ref{ass:coefficients} is used only to make the previous scalar message a sufficient summary of the public history.  Dropping it gives a weaker protocol whose rounds still reduce risk by $\mathbb E[(p_k-p_{k-1})^2]$, but without a closed-form per-round improvement.  Alternatively, sending the full message history in place of just $p_{k-1}$ restores the closed-form improvement $\alpha_k^2$, because the increment $p_k-p_{k-1}=\alpha_k S_k$ always reveals the next parity.
\end{remark}

\begin{remark}[Loss-adaptive stopping vs.\ exact recovery of $h^\star$]
\label{rem:stopping}
To achieve squared-loss excess at most $\varepsilon$ it suffices to run until $\sum_{j>k}\alpha_j^2\leq\varepsilon$, which can be much smaller than $K$. Exact recovery of $h^\star$ may instead require all $K$ rounds: if $\sum_{j\leq k}\alpha_j<\sum_{j>k}\alpha_j$ for every $k<K$, the unrevealed tail can flip the sign of $u_k+r_k$, so no $k<K$ prediction message determines $h^\star$.
\end{remark}

\subsection{Proof of Theorem~\ref{thm:action-stalls}}

\begin{proof}
For the parity algebra, let $\tilde p_k$ be the round-$k$ speaker's posterior under the action-channel transcript, and let $A_k:=2\,\mathbf 1\{\tilde p_k\geq 1/2\}-1=\sgn(\tilde p_k-1/2)\in\{-1,+1\}$ be the $\pm 1$ recoding of the round-$k$ action message. The claim in Theorem~\ref{thm:action-stalls} that every message equals $\mathbf 1\{b_1=+1\}$ is equivalent to $A_k=b_1$ for every $k$.

We show $A_k=b_1$ by induction.  At round $1$, Alice has no previous message and knows $b_1$; every $S_j$ with $j\geq 2$ contains Bob's unseen $b_2$ and vanishes in conditional expectation by Lemma~\ref{lem:vanishing}, so her posterior is $1/2+\alpha_1 b_1$ and $A_1=\sgn(\alpha_1 b_1)=b_1$.

Suppose all previous action messages equal $b_1$, so the public transcript reveals no information beyond the first bit. If Alice speaks at round $k$, she still lacks $b_2$, so the same conditional-expectation calculation gives $\tilde p_k=1/2+\alpha_1 b_1$ and $A_k=b_1$. If Bob speaks, he knows $b_2$; every $S_j$ with $j\geq 3$ contains the unseen bit $b_3$ and averages out, so his posterior is
\begin{equation*}
    \tilde p_k = \frac{1}{2}+\alpha_1 b_1+\alpha_2 b_1 b_2
    =
    \frac{1}{2}+b_1\big(\alpha_1+\alpha_2 b_2\big).
\end{equation*}
Since $\alpha_1>\alpha_2$, the factor $\alpha_1+\alpha_2 b_2$ is positive for both values of $b_2$, so $A_k=b_1$. Hence $A_k=b_1$ at every round and the induced decision does not change after round $1$.

For the constant-gap instance, take $K=3$, $\alpha_1=0.20$, $\alpha_2=0.15$, $\alpha_3=0.10$, which satisfy $\alpha_1>\alpha_2$ and $\alpha_1+\alpha_2+\alpha_3<1/2$. Theorem~\ref{thm:prediction-aggregates} gives $p_3=p^\star$. Meanwhile, the action message at every round equals $\mathbf 1\{b_1=+1\}$ (equivalently, $A_k=b_1$). On $E:=\{b_2=-1,b_3=+1\}$, which has probability $1/4$, the prefix parities are $S_1=b_1$ and $S_2=S_3=-b_1$, so $p^\star-1/2=b_1(\alpha_1-\alpha_2-\alpha_3)<0$; hence $h^\star=\mathbf 1\{p^\star\geq 1/2\}=\mathbf 1\{b_1=-1\}$, which disagrees with the action message on $E$. Thus $\Pr[\hat a^k\neq h^\star]\geq 1/4$ at every round.
\end{proof}

\section{Proofs for Section \ref{sec:calibration}}
\label{app:calibration}
Throughout this appendix section, we work in the chained-parity model of Section~\ref{sec:chained-parities} under Assumption~\ref{ass:coefficients}. We begin with the following observation:

\begin{observation}[Learning cross-calibrated forecasts] \label{obs:learning_cc}
    Let $Z$ be any information available at prediction time, such as a pair of communicated forecasts $Z=(\hat p_A^k,\hat p_B^\ell)$. Suppose the predictor class can represent the true conditional label probabilities. If the training loss is one whose best population prediction is the true conditional probability, such as cross-entropy or squared loss for binary probabilities, then training a predictor $h$ to predict $a^\star$ from $Z$ gives
    \[
        h^\star(Z)(a)=\Pr(a^\star=a\mid Z).
    \]
    Therefore the learned forecast $\tilde p=h^\star(Z)$ is calibrated even after conditioning on $Z$:
    $
        \Pr(a^\star=a\mid \tilde p,Z)=\tilde p(a)$
    almost surely. Moreover, if $Z=(W,\hat q)$ includes a previous communicated forecast $\hat q$, then also $\Pr(a^\star=a\mid \tilde p,\hat q)=\tilde p(a)$
    almost surely. Thus loss minimization on features that include the previous forecast learns a forecast that is cross-calibrated with respect to that forecast.
\end{observation}

\begin{proof}
By the assumption on the loss, the best population predictor outputs the true conditional label probabilities, so $h^\star(Z)(a)=\Pr(a^\star=a\mid Z)$. Let $\tilde p=h^\star(Z)$. Since $\tilde p$ is a function of $Z$,
\[
    \Pr(a^\star=a\mid \tilde p,Z)
    =
    \Pr(a^\star=a\mid Z)
    =
    \tilde p(a).
\]
If $Z=(W,\hat q)$, then averaging the last display over the remaining information $W$, while keeping $(\tilde p,\hat q)$ fixed, gives
\[
    \Pr(a^\star=a\mid \tilde p,\hat q)
    =
    \mathbb E[\Pr(a^\star=a\mid Z)\mid \tilde p,\hat q]
    =
    \mathbb E[\tilde p(a)\mid \tilde p,\hat q]
    =
    \tilde p(a).
\]
\end{proof}

We then expand the calibration claim of Theorem~\ref{thm:cross-cal-recovery} into self- and cross-calibration (Definitions~\ref{def:calibration} and~\ref{def:cross_calibration}), and quantify the excess squared loss from using raw reports $\hat p_k$ in place of $p_k$.

We first consider the direct-use benchmark in which the receiver treats the raw report $\hat p_k$ itself as the round-$k$ prediction. This benchmark isolates the squared-loss cost of reporting a biased scalar before applying any post-hoc calibration or projection rule.

\begin{proposition}[Self-calibration removes global miscalibration]
\label{prop:self-calibration}
If $\lambda=0$ and $c>0$ in~\eqref{eq:biased-report}, then self-calibration already recovers the ideal scalar: $\Pr(a^\star=1\mid \hat p_k)=p_k$. Using raw $\hat p_k$ in place of $p_k$ incurs excess squared loss
\begin{equation*}
    \mathbb E[(\hat p_k-p_k)^2]
    =(c-1)^2\,\mathbb E\left[\big(p_k-\tfrac{1}{2}\big)^2\right]
    =(c-1)^2\sum_{j\leq k}\alpha_j^2.
\end{equation*}
\end{proposition}

\begin{proof}
When $\lambda=0$, $\hat p_k=1/2+c(p_k-1/2)$ is invertible: $p_k=1/2+(\hat p_k-1/2)/c$. Since $p^\star=p_k+r_k$ and $\hat p_k$ is a function of $b_1,\ldots,b_k$, Lemma~\ref{lem:vanishing} gives $\mathbb E[r_k\mid \hat p_k]=0$, so $\Pr(a^\star=1\mid \hat p_k)=\mathbb E[p^\star\mid \hat p_k]=p_k$. The raw-report loss is $\mathbb E[(\hat p_k-p_k)^2]=(c-1)^2\mathbb E[(p_k-1/2)^2]=(c-1)^2\sum_{j\leq k}\alpha_j^2$.
\end{proof}

\begin{proposition}[Cross-calibration removes previous-message bias]
\label{prop:cross-calibration}
For any $c>0$ and $\lambda\in\mathbb R$ in~\eqref{eq:biased-report}, the conditional label distribution given the raw report and the previous agent's report recovers the ideal scalar:
\[
    \Pr(a^\star=1\mid \hat p_k,p_{k-1})=p_k.
\]
Thus a post-hoc forecast learned from $(\hat p_k,p_{k-1})$ is cross-calibrated with respect to $p_{k-1}$ and outputs $p_k$. Using raw $\hat p_k$ in place of $p_k$ incurs excess squared loss
\begin{equation*}
    \mathbb E[(\hat p_k-p_k)^2]
    =(c-1)^2\sum_{j\leq k}\alpha_j^2
    +\big(2(c-1)\lambda+\lambda^2\big)\sum_{j<k}\alpha_j^2.
\end{equation*}
\end{proposition}

\begin{proof}
For any $\lambda$, observing $(\hat p_k,p_{k-1})$ determines
\begin{equation*}
    p_k=\tfrac{1}{2}+\tfrac{1}{c}\left[\big(\hat p_k-\tfrac{1}{2}\big)-\lambda\big(p_{k-1}-\tfrac{1}{2}\big)\right],
\end{equation*}
so $p^\star=p_k+r_k$ with $\mathbb E[r_k\mid \hat p_k,p_{k-1}]=0$ by Lemma~\ref{lem:vanishing}. Hence $\Pr(a^\star=1\mid \hat p_k,p_{k-1})=\mathbb E[p^\star\mid \hat p_k,p_{k-1}]=p_k$. If the raw report is used directly then $\hat p_k-p_k=(c-1)(p_k-1/2)+\lambda(p_{k-1}-1/2)$. With $p_k-1/2=u_k$ and $p_{k-1}-1/2=u_{k-1}=u_k-\alpha_k S_k$, cross-term cancellation via Lemma~\ref{lem:vanishing} gives $\mathbb E[u_k^2]=\sum_{j\leq k}\alpha_j^2$ and $\mathbb E[u_{k-1}^2]=\mathbb E[u_k u_{k-1}]=\sum_{j<k}\alpha_j^2$, yielding the stated identity.

\end{proof}

Proposition~\ref{prop:cross-calibration} is exactly the chained-parity instance of Observation~\ref{obs:learning_cc}: learning $\Pr(a^\star=1\mid \hat p_k,p_{k-1})$ from labelled transcripts produces a cross-calibrated forecast that equals $p_k$, and hence recovers Theorem~\ref{thm:prediction-aggregates} regardless of the raw report's confidence distortion ($c$) or anchoring to the previous message ($\lambda$). This proves Theorem~\ref{thm:cross-cal-recovery}.

\noindent\textbf{Round-by-round behavior under projection.}
Propositions~\ref{prop:self-calibration} and~\ref{prop:cross-calibration} say that, given enough labelled data, calibration recovers the ideal scalar $p_k$ from a biased report. We now analyze a different receiver model: a receiver that does not know the bias model and therefore treats $\hat p_k$ as if it were the true posterior. When $\hat p_k$ is not a possible message under any valid speaker posterior, the receiver rounds $\hat p_k$ to the closest valid message before updating. Formally, if $S$ is the set of valid speaker messages, the receiver's Bayesian update is computed with respect to
\begin{equation}
\label{eq:projection-rule}
    \bar p_k \;\in\; \arg\min_{s\in S}\,|\hat p_k - s|.
\end{equation}
The set $S$ should be read as the \emph{receiver-consistent} valid set: the posterior values that could have been sent by the speaker among worlds consistent with the receiver's own private observation and the public transcript so far, which avoids assigning positive probability to a message the receiver already knows to be impossible.

Under this rule, bias is harmless until it pushes $\hat p_k$ across the midpoint between two adjacent valid messages, at which point the receiver decodes the wrong latent bit and pays a discrete cost. We work with the original biased-report model of~\eqref{eq:biased-report}, with the same $(c,\lambda)$ across rounds, and trace what happens round by round.

\noindent\textbf{Round $1$.}
Alice observes $b_1$ and her ideal posterior is $p_1=\tfrac12+\alpha_1 b_1$. From Bob's view the valid messages are $\{\tfrac12-\alpha_1,\tfrac12+\alpha_1\}$, symmetric around $p_0=\tfrac12$ at distance $\alpha_1$. Since $p_0=\tfrac12$, the anchoring term in~\eqref{eq:biased-report} vanishes and $\hat p_1-\tfrac12=c\alpha_1 b_1$ has the same sign as $b_1$ for every $c>0$. Nearest-valid projection therefore recovers $b_1$, and the round-$1$ message conveys the same information it would under a calibrated speaker.

\noindent\textbf{Round $2$.}
Conditional on correct round-$1$ decoding, Bob now sees $p_1=\tfrac12+\alpha_1 b_1$ as a public scalar and observes his own bit $b_2$. His ideal posterior is $p_2=p_1+\alpha_2 b_1 b_2$, and from Alice's view the valid round-$2$ messages are $\{p_1-\alpha_2,p_1+\alpha_2\}$, symmetric around $p_1$ at distance $\alpha_2$. Nearest-valid projection of $\hat p_2$ recovers $p_2$ iff $\hat p_2$ lies on the same side of $p_1$ as $p_2$ (with equality corresponding to a tie). Unlike at round $1$, the biased report now picks up an anchoring contribution from $p_{k-1}=p_1$, and this is the source of failure.

\begin{proposition}[Decoding failure at round 2]
\label{prop:projected-round-2}
Under the biased reports of~\eqref{eq:biased-report} with $c>0$, conditional on correct round-$1$ decoding, nearest-valid projection at round $2$ decodes incorrectly on the event $\{b_2=-\sgn(c+\lambda-1)\}$, of probability $1/2$, whenever
\begin{equation*}
    |c+\lambda-1|\,\alpha_1 \;>\; c\,\alpha_2.
\end{equation*}
A wrong decode flips $\alpha_2 S_2$ in the receiver's projected posterior. Conditional on the wrong-decode event, this contributes squared-loss excess $4\alpha_2^2$ relative to the ideal round-$2$ posterior; averaged over $b_2$ conditional on correct round-$1$ decoding, the contribution is $2\alpha_2^2$.
\end{proposition}
\begin{proof}
Conditional on correct round-$1$ decoding, $p_1=\tfrac12+\alpha_1 b_1$ and $p_2=p_1+\alpha_2 b_1 b_2$, so~\eqref{eq:biased-report} gives
\begin{equation*}
    \hat p_2-p_1 = c\alpha_2 b_1 b_2 + (c+\lambda-1)\alpha_1 b_1.
\end{equation*}
Correct decoding requires this to have the same sign as $b_1 b_2$; multiplying by $b_1 b_2$ gives the equivalent condition $c\alpha_2 + (c+\lambda-1)\alpha_1 b_2>0$, which fails on the unfavorable sign $b_2=-\sgn(c+\lambda-1)$ exactly when $|c+\lambda-1|\alpha_1> c\alpha_2$ (with equality corresponding to a tie in nearest-valid projection).  On the wrong-decode event, projection flips $\alpha_2 S_2$ to $-\alpha_2 S_2$, so the projected posterior differs from the ideal round-$2$ posterior by $2\alpha_2 S_2$. By the squared-loss identity used in Theorem~\ref{thm:prediction-aggregates}, the conditional excess is $(2\alpha_2)^2=4\alpha_2^2$. Since the wrong-decode event has probability $1/2$ conditional on correct round-$1$ decoding, its expected contribution is $2\alpha_2^2$.
\end{proof}

This parallels Theorem~\ref{thm:action-stalls}: action communication transmits $b_1$ correctly and then stalls; affine-biased probability communication also transmits $b_1$ correctly under projection, but begins to lose $b_2$ once the imbalance condition $|c+\lambda-1|\alpha_1>c\alpha_2$ holds. The round-$2$ projection failure disappears when $c+\lambda=1$, while true calibration of the affine report corresponds to the stronger condition $c=1$ and $\lambda=0$. The two parts of the appendix are thus complementary: Propositions~\ref{prop:self-calibration}--\ref{prop:cross-calibration} say that post-hoc calibration suffices to recover the ideal scalar $p_k$, while Proposition~\ref{prop:projected-round-2} identifies a regime in which an uncalibrated projection-based receiver flips the second latent bit.

\section{Experimental details} \label{app:exp_details}

Here, we provide additional information on reproducibility of our experiments.

\subsection{Task design}

\textbf{Chained Parities}~~~For the parity task, we construct our dataset as prescribed by the theory in Section~\ref{sec:chained-parities}. We generate bitstrings of length 30 and choose a total of $30$ corresponding coefficients. We set the first four values to $[0.2, 0.15, 0.1, 0.02]$ and assign the remaining $26$ values equally so that $\sum_j \alpha_j = 0.49$. The bits of each datapoint are split disjointly between the two agents in alternating order, and the binary label is sampled from $P(Y{=}1 \mid b) = \tfrac{1}{2} + \sum_j \alpha_j s_j$, where $s_j$ is the $j$-th prefix product of the hidden signs. Then, we generate a total of two million datapoints which we split into a $90\%$ train and $10\%$ test split. All results are reported on the test split.

\textbf{Mazes}~~~For the maze task, we construct mazes similar to those introduced by~\citep{davidson2025collaboration}. We consider mazes of size $6\times6$. For evaluation purposes, we only consider mazes with a single correct path. The wall density is $30\%$. Each agent sees exactly half of the maze cells, assigned uniformly at random, and the agents receive disjoint masks so that each maze is always solvable from the joint information. An example maze is shown below. Again, we generate two million datapoints in total, split the data into a $90\%$ train and $10\%$ test split, and report our results on the test set. The data contains one example per step along the unique solution: the agent's masked grid is rendered with the current position marked, and the correct next move is used as the label. There are four answer tokens that can represent the true label: up, down, left, and right.

\begin{figure}[htbp]
    \centering

    \begin{subfigure}{0.31\textwidth}
        \centering
        $x$:
        \begin{verbatim}
            @##.#.
            .##...
            ...#..
            .#.##.
            #..###
            .#...*
        \end{verbatim}
        \label{fig:maze1}
    \end{subfigure}
    \hfill
    \begin{subfigure}{0.31\textwidth}
        \centering
        $x_A$:
        \begin{verbatim}
            @#??#?
            ?#??.?
            ?.??.?
            .?.##?
            #..##?
            .#.??*
        \end{verbatim}
        \label{fig:maze2}
    \end{subfigure}
    \hfill
    \begin{subfigure}{0.31\textwidth}
    \centering
    $x_B$:
    \begin{verbatim}
            @?#.?.
            .?#.?.
            .?.#?.
            ?#???.
            ?????#
            ???..*
    \end{verbatim}
    \label{fig:maze3}
    \end{subfigure}

    \caption{Example of full maze and agents' private observations.}
    \label{fig:maze-example}
\end{figure}

\subsection{Transformer experiments}

For both tasks we use a small, character-level GPT based on the NanoGPT
implementation~\citep{karpathy2022nanogpt}: $4$ transformer blocks, $6$
attention heads, embedding dimension $240$, and a causal mask. Each agent
maintains its own per-task vocabulary derived from the raw character set of
its inputs, so all sequences are tokenized at the character level. The
context length is set to $\texttt{prefix\_size}+1$, where the prefix is the
serialized example up to (but not including) the answer; any annotations
appended by a collaborator (see below) extend the prefix accordingly. We
treat each task as next-token prediction with the loss measured only at the
answer position(s).

\textbf{Per-agent inputs.}~~~For chained parity, agent $0$ ("Bob") sees the
odd-indexed bits and agent $1$ ("Alice") sees the even-indexed bits; the
unobserved positions are replaced with a mask token, and the label bit is
appended after an \texttt{=} delimiter. For mazes, each agent's input is
the flattened $6\times 6$ grid restricted to its assigned half of the
cells, with the current position marked by \texttt{@}, followed by
\texttt{=} and the next move from $\{d,r,u,l\}$. A single maze with path
length $\ell$ contributes $\ell$ training examples, each predicting the
next move from the current position.

\textbf{Conversation loop.}~~~We instantiate our protocol as iterative
training across rounds: in round $r$, agent $r \bmod K$ (with $K$ the
number of agents) trains a fresh transformer from scratch on its currently
annotated dataset. After a fixed number of optimizer steps, the trained
agent runs over its inputs and emits a message for each
example, which is then appended to the prefix of the next agent's inputs in round $r{+}1$. Because each round trains a freshly initialized model,
any improvement in joint performance must come through the appended
messages rather than transferred weights.

\textbf{Optimization.}~~~Each agent is trained from scratch with AdamW
($\beta_1=0.9$, $\beta_2=0.99$, weight decay $0.1$, gradient clip $1.0$)
in bfloat16 mixed precision on a single GPU, with batch size $1024$ and no
gradient accumulation. The learning rate uses linear warmup over $100$
steps to a peak of $10^{-4}$, followed by cosine decay to $10^{-5}$ over
the remainder of training. We run $2000$ iterations per round on chained
parity ($5$ rounds) and $50{,}000$ iterations per round on mazes
($4$ rounds); in both cases we keep the checkpoint with the best
validation loss within each round.

\subsection{LLM experiments} \label{app:llm_exp_details}

The LLM experiments follow the same per-round procedure as our
transformer experiments, with each agent now replaced by a
pretrained LLM queried through the API rather than a model trained from
scratch. Unless otherwise noted we use temperature $1.0$, run $4$ rounds of conversation
per maze, and evaluate on $N=30$ mazes, with 10 rollouts each, drawn from the held-out test split.
At each round, the active agent receives a prompt containing (i) a
legend-annotated rendering of its own masked $6\times 6$ maze, (ii) the
sequence of moves taken so far (from which the current position can be
inferred), and (iii) the other agent's most recent response from the
previous round. We consider three communication settings, which differ
only in what part of the partner's response is appended to the next
prompt: the full free-form message, a single move token in
$\{d,r,u,l\}$, or a probability vector over the four
moves. Probabilities can either be elicited
directly from the model in the response, post-calibrated from the verbalized answer, or obtained from rollouts. The agent is
to produce a single move
after the delimiter \texttt{\textasciitilde}, and -- in the verbalized
setting -- a probability vector after the delimiter \texttt{\%}. The
joint trajectory advances by one step after each conversation of $K$ rounds; a single incorrect move
terminates the maze, and we report the resulting per-maze success rate.

\subsubsection{Prompts}

The prompt we use to query the LLM at each step is:

\begin{lstlisting}[basicstyle=\ttfamily\footnotesize,breaklines=true,frame=single,columns=fullflexible]
Task: You are going to play the collaborative maze game together with another agent. The maze consists of a grid with walls and a goal. Your task is to jointly determine the next move on the path. You and the other agent will take that action together. You will each get your own map of the same maze, with some coordinates hidden. Because of the hidden coordinates, you will need to communicate with the other agent to share information about the maze and coordinate your next move. Both agents together have enough information to solve the maze, so you do not need to explore.
Rules:
- You can only move to adjacent cells [down(d), right(r), up(u), left(l)].
- You can not move diagonally or through walls.
- Once you make an incorrect move, you lose.
Here is **your** map of the maze with a legend of the symbols:
@?#.?.
.?#.?.
.?.#?.
?#???.
?????#
???..*
Legend:
@ - Current Position
* - Goal Position
. - Path
# - Wall
? - Hidden Cell

You must answer with one of d, r, u, l. Put your final answer after the delimiter ~. For example, if your final answer is d, your response should contain ~d.

Finally, give the probabilities you think each move d, r, u, l is correct after the delimiter %. For example, if you think the probability of d, r, u, l is 0.8, 0.1, 0.05, 0.05, your response should contain %[0.8, 0.1, 0.05, 0.05]. Make sure it is a valid probability vector, i.e. the sum of the probabilities is 1.

Your moves are timed for speed so hurry up!
\end{lstlisting}

For our experiment communicating full responses (Figure \ref{fig:full-response}), we slightly modify the above prompt for GPT-5.4-nano to ensure its response includes content beyond the elicited probabilities (otherwise, while it generates reasoning in the background, its reasoning is hidden). The modified prompt is:

\begin{lstlisting}[basicstyle=\ttfamily\footnotesize,breaklines=true,frame=single,columns=fullflexible]
Task: You are going to play the collaborative maze game together with another agent. The maze consists of a grid with walls and a goal. Your task is to jointly determine the next move on the path. You and the other agent will take that action together. You will each get your own map of the same maze, with some coordinates hidden. Because of the hidden coordinates, you will need to communicate with the other agent to share information about the maze and coordinate your next move. Both agents together have enough information to solve the maze, so you do not need to explore.
Rules:
- You can only move to adjacent cells [down(d), right(r), up(u), left(l)].
- You can not move diagonally or through walls.
- Once you make an incorrect move, you lose.
Here is **your** map of the maze with a legend of the symbols:
@?#.?.
.?#.?.
.?.#?.
?#???.
?????#
???..*
Legend:
@ - Current Position
* - Goal Position
. - Path
# - Wall
? - Hidden Cell

First, provide what you would like to communicate to the other agent. Then, you must answer with one of d, r, u, l. Put your final answer after the delimiter ~. For example, if your final answer is d, your response should contain ~d.

Finally, give the probabilities you think each move d, r, u, l is correct after the delimiter %. For example, if you think the probability of d, r, u, l is 0.8, 0.1, 0.05, 0.05, your response should contain %[0.8, 0.1, 0.05, 0.05]. Make sure it is a valid probability vector, i.e. the sum of the probabilities is 1.

Your moves are timed for speed so hurry up!
\end{lstlisting}

Finally, the prompt we use for the full-information solo baseline (Figure \ref{fig:compare-comm-panel}) is: 

\begin{lstlisting}[basicstyle=\ttfamily\footnotesize,breaklines=true,frame=single,columns=fullflexible]
Task: You are going to play a maze game. The maze consists of a grid with walls and a goal. Your task is to navigate through the maze, avoiding walls, to reach the goal.
Rules:
- You can only move to adjacent cells [down(d), right(r), up(u), left(l)].
- You can not move diagonally or through walls.
- Once you make an incorrect move, you lose.
Here is **your** map of the maze with a legend of the symbols:
@...#.
.#.##.
.#...#
.###..
#...#.
..#.#*
Legend:
@ - Current Position
* - Goal Position
. - Path
# - Wall
? - Hidden Cell

You must answer with one of d, r, u, l. Put your final answer after the delimiter ~. For example, if your final answer is d, your response should contain ~d.

Finally, give the probabilities you think each move d, r, u, l is correct after the delimiter %. For example, if you think the probability of d, r, u, l is 0.8, 0.1, 0.05, 0.05, your response should contain %[0.8, 0.1, 0.05, 0.05]. Make sure it is a valid probability vector, i.e. the sum of the probabilities is 1.

Your moves are timed for speed so hurry up!
\end{lstlisting}

\subsubsection{Conversation Calibration Loss}

In Figure~\ref{fig:cal-panel}, we compute conversation calibration loss as
follows. At round $k$ the model is evaluated on $N$ examples (one per move
per maze), indexed $i=1,\dots,N$. Given $D$ possible actions, for each example: let
$\tilde p_{k,i}\in[0,1]^{D}$ be the model's predicted action distribution at
round $k$, $\hat p_{k-1,i}\in[0,1]^{D}$ be the collaborator's communicated
distribution at the previous round $k-1$, and $y_i\in\{0,1\}^{D}$ be the
one-hot encoding of the true action.

Discretize $[0,1]$ into $B$ equal-width bins, $I_b=[\tfrac{b-1}{B},\tfrac{b}{B})$
for $b=1,\dots,B$ (with $I_B$ closed at $1$). For a vector $p\in[0,1]^D$ and a
bin multi-index $\mathbf b=(b_1,\dots,b_D)\in\{1,\dots,B\}^D$, write the cell
indicator as a product over the $D$ coordinates,
\[
  \mathbf 1\{p\in I_{\mathbf b}\}
  \;=\;\prod_{d=1}^{D}\mathbf 1\!\big\{p^{(d)}\in I_{b_d}\big\}.
\]
A \emph{cross bin} is a pair $(\mathbf b,\mathbf b')$ that jointly bins the
model's own distribution and the collaborator's distribution. We use hard binning --- that is, the weight on cross bins $(\mathbf b,\mathbf b')$ for example $i$ is simply the indicator of if it lies in the bins:
\[
  w_{k,i}(\mathbf b,\mathbf b')
  \;=\;\mathbf 1\{\tilde p_{k,i}\in I_{\mathbf b}\}\;
        \mathbf 1\{\hat p_{k-1,i}\in I_{\mathbf b'}\}\;\in\{0,1\}.
\]
Let $N_{\mathbf b,\mathbf b'}=\sum_{i=1}^{N} w_{k,i}(\mathbf b,\mathbf b')$ be the
bin count, and let
\[
  \bar p_{\mathbf b,\mathbf b'}
   =\frac{1}{N_{\mathbf b,\mathbf b'}}\sum_{i=1}^{N}
     w_{k,i}(\mathbf b,\mathbf b')\,\tilde p_{k,i},
  \qquad
  \bar y_{\mathbf b,\mathbf b'}
   =\frac{1}{N_{\mathbf b,\mathbf b'}}\sum_{i=1}^{N}
     w_{k,i}(\mathbf b,\mathbf b')\,y_i
\]
be the bin-averaged prediction and empirical outcome frequency. The
conversation calibration loss at round $k$ is the average weighted $\ell_2$ gap
between the two:
\[
  \text{ConvCal}_k
  \;=\;\frac{1}{N}\sum_{\mathbf b,\mathbf b'}
        N_{\mathbf b,\mathbf b'}\,
        \big\|\bar p_{\mathbf b,\mathbf b'}-\bar y_{\mathbf b,\mathbf b'}\big\|_2
  \;=\;\frac{1}{N}\sum_{\mathbf b,\mathbf b'}
        \Big\|\sum_{i=1}^{N} w_{k,i}(\mathbf b,\mathbf b')
              \big(\tilde p_{k,i}-y_i\big)\Big\|_2 ,
\]
The second equality follows since
$\bar p_{\mathbf b,\mathbf b'}-\bar y_{\mathbf b,\mathbf b'}
=\tfrac{1}{N_{\mathbf b,\mathbf b'}}\sum_i w_{k,i}(\mathbf b,\mathbf b')
(\tilde p_{k,i}-y_i)$, so the factor $N_{\mathbf b,\mathbf b'}$ cancels. At
round $k=1$ there is no collaborator, so we bin on $\tilde p_{1,i}$ alone
(i.e. drop $\mathbf b'$ and the $\hat p$ indicator), recovering the standard
(self) calibration loss. In our results, we report conversation calibration losses using $B=10$.


\section{Examples of Multi-Round Communication in Mazes}\label{app:hard-examples}

Here we give examples of ``hard'' moves in mazes --- those that require multiple rounds of communication --- and how our trained transformer models converge to the correct move when iteratively exchanging probabilities. 

\begin{figure}[htbp]
    \centering

    \begin{subfigure}[t]{\textwidth}
        \centering

        \vspace{0.5em}

        \begin{subfigure}[t]{0.31\textwidth}
            \centering
            $x$:
            \begin{verbatim}
            @...#.
            .#.##.
            .#...#
            .###..
            #...#.
            ..#.#*
            \end{verbatim}
        \end{subfigure}
        \hspace{-80pt}
        \begin{subfigure}[t]{0.31\textwidth}
            \centering
            $x_A$:
            \begin{verbatim}
            @...??
            ?#????
            .?.?.?
            ?#?#??
            ?..??.
            ..#?#*
            \end{verbatim}
        \end{subfigure}
        \hspace{-80pt}
        \begin{subfigure}[t]{0.31\textwidth}
            \centering
            $x_B$:
            \begin{verbatim}
            @???#.
            .?.##.
            ?#?.?#
            .?#?..
            #??.#?
            ???.?*
            \end{verbatim}
        \end{subfigure}

        \vspace{0.5em}

        \includegraphics[width=0.42\textwidth]{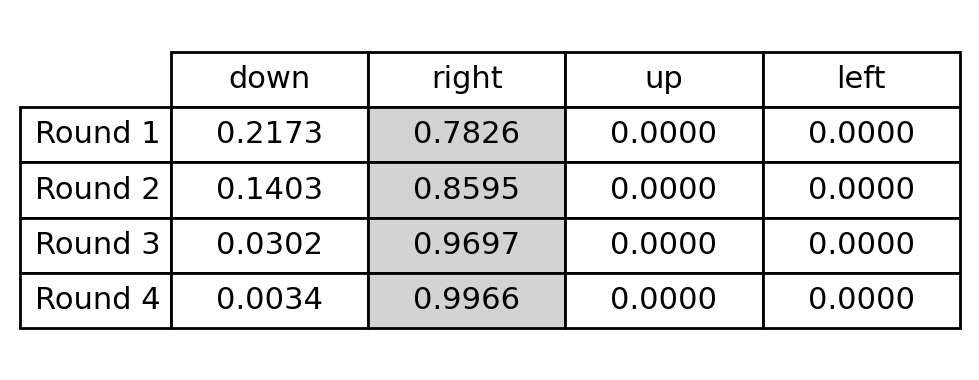}
        \caption{Hard maze example 1.}
        \label{fig:hard-maze-1}
    \end{subfigure}

    \vspace{1.2em}

    \begin{subfigure}[t]{\textwidth}
        \centering

        \vspace{0.5em}

        \begin{subfigure}[t]{0.31\textwidth}
            \centering
            $x$:
            \begin{verbatim}
            ..#.#.
            #.#...
            ..@..#
            .#.###
            #.....
            .##.#*
            \end{verbatim}
        \end{subfigure}
        \hspace{-80pt}
        \begin{subfigure}[t]{0.31\textwidth}
            \centering
            $x_A$:
            \begin{verbatim}
            ..????
            ?????.
            ..@?.?
            ????#?
            #..???
            .?#?#*
            \end{verbatim}
        \end{subfigure}
        \hspace{-80pt}
        \begin{subfigure}[t]{0.31\textwidth}
            \centering
            $x_B$:
            \begin{verbatim}
            .?#.#.
            #.#..?
            ??@.?#
            .#.#?#
            ???...
            ?#?.?*
            \end{verbatim}
        \end{subfigure}

        \vspace{0.5em}

        \includegraphics[width=0.42\textwidth]{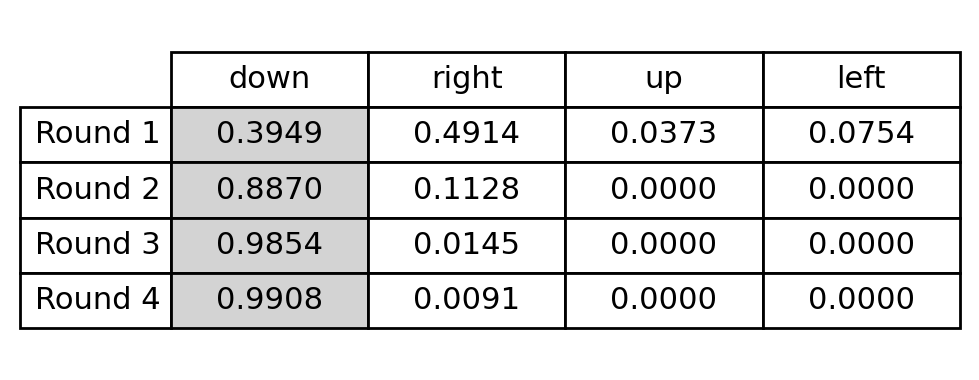}
        \caption{Hard maze example 2.}
        \label{fig:hard-maze-2}
    \end{subfigure}

    \caption{Two examples of hard moves in mazes. For each maze, the panels show the full maze $x$, agents' partially observed mazes $x_A$ and $x_B$ with the current position denoted by @, and the probabilistic predictions made by our trained transformer models on this move over 4 rounds (details are in Appendix~\ref{app:exp_details}). Agent $A$ communicates on odd rounds and Agent $B$ communicates on even rounds. The correct answer is shaded.}
    \label{fig:hard-mazes}
\end{figure}

Figure \ref{fig:hard-mazes} shows two examples where both agents are initially uncertain about the next move. In the first example (Figure \ref{fig:hard-maze-1}), agent $A$ assigns 78\% probability to moving right in the first round, which is the correct move. Subsequent rounds progressively reinforce this prediction, and by the fourth round, agent $B$ is almost completely certain that this is the correct next move. In the second example (Figure \ref{fig:hard-maze-2}), agent $A$ is initially fairly uncertain and assigns most of its probability mass to moving right, which is incorrect. Agent $B$ corrects this in the second round by assigning 89\% probability to moving down. Agent $A$ then resolves almost all of its uncertainty and ends up assigning 99\% probability to moving down. Note that in all other moves in these mazes, two rounds suffice: one agent exactly knows the next move.

\section{Limitations}
\label{ap:limit}
Our work studies one specific axis of multi-agent LLM collaboration---the communication protocol when information is partitioned between two agents---and we make no claim to have addressed the full complexity of this space, including issues of task decomposition, trust between agents, and prompt sensitivity. On the theory side, while we establish a clean separation between action and probability communication in the chained parity model, a complete characterization of when and why this gap arises in general remains open. The maze task, though a useful controlled testbed, is stylized relative to real collaborative tasks. Finally, our setting is restricted to two agents; extending the framework to networks of more than two agents is a natural direction for future work.

\section{Compute Resources} \label{app:compute}

Our experiments were conducted partially on a local desktop computer and on a compute cluster with various types of graphical processing units (GPUs), including Nvidia RTX3090 and Nvidia A6000 GPUs. Training a single parity experiment can be done on a household-grade GPU in less than one hour time. The training of the maze transformers took approximately 6-12 hours per run depending on the GPU. Our experiments thus amount to approximately $300-400$ GPU hours. This is an estimate for the experiments included in the paper and does not reflect potential developing and debugging runtime. All LLM experiments were conducted directly via OpenAI's public API and did not require local computational resources beyond a standard laptop.

\end{document}